%% file: main.tex
\documentclass[sigconf]{acmart}

\usepackage{booktabs}
\usepackage{subfig}
\usepackage{bm}
\usepackage{balance}
\usepackage{tikz}
\usepackage{colortbl}
\usepackage{subfig}
\usepackage{graphicx}
\usetikzlibrary{positioning}
\usepackage{enumitem}

\usepackage{amsthm}
\usepackage{amsmath}
\usepackage{amsfonts}
\usepackage{bbm}
\usepackage{chngcntr}
\usepackage{apptools}

\newtheorem{theorem}{Theorem}
\newtheorem*{theorem*}{Theorem}

\newtheorem*{lemma*}{Lemma}
\newtheorem{claim}{Claim}
\newtheorem*{claim*}{Claim}
\newtheorem{definition}{Definition}
\newtheorem{proposition}{Proposition}
\newtheorem*{proposition*}{Proposition}

\DeclareMathOperator*{\argmax}{arg\,max}

\newcommand{\xendcg}{\textsc{xe}_\textsc{ndcg}}
\newcommand{\xendcgmart}{\textsc{xe}_\textsc{ndcg}\textsc{mart}}
\newcommand{\lambdamart}{\lambda\textsc{mart}}

\setcopyright{acmcopyright}
\copyrightyear{2021}
\acmYear{2021}
\acmConference[WWW '21]{Proceedings of the Web Conference 2021}{April 19--23, 2021}{Ljubljana, Slovenia}
\acmBooktitle{Proceedings of the Web Conference 2021 (WWW '21), April 19--23, 2021, Ljubljana, Slovenia}
\acmPrice{}
\acmDOI{10.1145/3442381.3449794}
\acmISBN{978-1-4503-8312-7/21/04}

\begin{document}
\title{An Alternative Cross Entropy Loss for Learning-to-Rank}

\author{Sebastian Bruch}
\authornote{This work was carried out at Google Research.}
\affiliation{
  \institution{National Institutes of Health}
}
\email{sebastian.bruch@nih.gov}


\begin{abstract}
Listwise learning-to-rank methods form a powerful class of ranking algorithms that are widely adopted in applications such as information retrieval. These algorithms learn to rank a set of items by optimizing a loss that is a function of the entire set---as a surrogate to a typically non-differentiable ranking metric. Despite their empirical success, existing listwise methods are based on heuristics and remain theoretically ill-understood. In particular, none of the empirically successful loss functions are related to ranking metrics. In this work, we propose a cross entropy-based learning-to-rank loss function that is theoretically sound, is a convex bound on NDCG---a popular ranking metric---and is consistent with NDCG under learning scenarios common in information retrieval. Furthermore, empirical evaluation of an implementation of the proposed method with gradient boosting machines on benchmark learning-to-rank datasets demonstrates the superiority of our proposed formulation over existing algorithms in quality and robustness.
\end{abstract}

%
%



\maketitle

\input{introduction}
\input{related_work}
\input{preliminaries}
\input{proposed_method}
\input{experiments}
\balance
\input{conclusion}

\section{Acknowledgements}
I extend my gratitude to Don Metzler for his in-depth review of an early draft of this work, and Qingyao Ai and Branislav Kveton for their feedback that led to much-needed clarifications. Special thanks go to Honglei Zhuang for allowing me to use him as a sounding board and for many engaging conversations on deeply technical aspects of the work. Finally, I am grateful to Masrour Zoghi and Ben Carterette for their thoughts on the experimental setup.

\appendix
\renewcommand{\thesection}{\Alph{section}}
\input{appendix}

\newpage
\balance
\bibliographystyle{ACM-Reference-Format}
\bibliography{main} 

\end{document}

%% file: introduction.tex
\section{Introduction} \label{sec:introduction}
Learning-to-rank is a central problem in a range of applications including web search, recommendation systems, and question answering. The task is to learn a function that, conditioned on some context, arranges a set of items into an ordered list so as to maximize a given metric. In this work, without loss of generality, we take search as an example where a set of documents (items) are ranked by their relevance to a query (context).

Rather than directly working with permutations, learning-to-rank methods typically approach the ranking problem as one of ``score and sort.'' The objective is then to learn a ``scoring'' function to estimate query-document relevance. Subsequently, they sort documents in decreasing relevance to form a ranked list. Ideally, the resulting ranked list should maximize a ranking metric.

Popular ranking metrics are instances of the general class of \emph{conditional linear rank statistics}~\cite{Clemencon:NeurIPS:2008} that summarize the Receiver Operator Characteristic (ROC) curve. Of particular interest are the ranking statistics that care mostly about the leftmost portion of the ROC curve, corresponding to the top of the ranked list. Mean Reciprocal Rank and Normalized Discounted Cumulative Gain~\cite{jarvelin2002cumulated} are two such metrics that are widely used in information retrieval.

Ranking metrics, as functions of learning-to-rank scores, are flat almost everywhere; a small perturbation of scores is unlikely to lead to a change in the metric. This property poses a challenge for gradient-based optimization algorithms, making a direct optimization of ranking metrics over a complex hypothesis space infeasible. Addressing this challenge has been the focus of a large body of research~\cite{liu2009learning}, with most considering smooth loss functions as surrogates to metrics.

The majority of existing loss functions~\cite{cao2007learning,burges2005learning,burges2010ranknet,xia2008listwise,joachims2006training}, however, are only loosely related to ranking metrics such as NDCG. ListNet~\cite{cao2007learning}, as an example, projects labels and scores onto the probability simplex and minimizes the cross-entropy between the resulting distributions. LambdaMART~\cite{burges2010ranknet,wu2010adapting} (denoted as $\lambda\textsc{mart}$) forgoes the loss function altogether and heuristically formulates the gradients.

The heuristic nature of learning-to-rank loss functions and a lack of theoretical justification for their use have hindered progress in the field. While $\lambda\textsc{mart}$ remains the state-of-the-art method to date, the fact that its loss function, presumed to be smooth, is unknown makes a theoretical analysis of the algorithm difficult. Empirical improvements over existing methods remain marginal for similar reasons.

In this work, we are motivated to help close the gap above. To that end, we present a construction of the cross-entropy loss which we dub $\xendcg$, that is only slightly different from the ListNet loss, but that enjoys strong theoretical properties. In particular, we prove that our construction is a convex bound on NDCG, thereby lending credence to its optimization for the purpose of learning ranking functions. Furthermore, we show that the generalization error of $\xendcg$ compares favorably with that of $\lambda\textsc{mart}$'s. Experiments on benchmark learning-to-rank datasets further reveal the empirical superiority of our proposed method. We anticipate the theoretical soundness of our method and its strong connection to ranking metrics enable future research and progress.

Our contributions can be summarized as follows:
\begin{itemize}[leftmargin=*]
    \item We present a cross entropy-based loss function, dubbed $\xendcg$, for learning-to-rank and prove that it is a convex bound on negative (translated and log-transformed) mean NDCG;
    \item We compare model complexity between $\lambdamart$ and $\xendcg$;
    \item We formulate an approximation to the inverse Hessian for $\xendcg$ for optimization with second-order methods; and,
    \item We optimize $\xendcg$ to learn Gradient Boosted Regression Trees (denoted by $\xendcgmart$) and compare its performance and robustness with $\lambdamart$ on benchmark learning-to-rank datasets through extensive experiments.
\end{itemize}

This article is organized as follows. Section~\ref{sec:related_work} reviews existing work on learning-to-rank. In Section~\ref{sec:preliminaries}, we introduce our notation and formulate the problem. Section~\ref{sec:proposed_method} presents a detailed description of our proposed learning-to-rank loss function and examines its theoretical properties, including a comparison of bounds on the generalization error. We empirically evaluate our method and report our findings in Section~\ref{sec:experiments}. We conclude this work in Section~\ref{sec:conclusion}.

%% file: related_work.tex
\section{Related Work} \label{sec:related_work}
A large class of learning-to-rank methods attempt to optimize pairwise misranking error---a popular ranking statistic in many prioritization problems---by learning to correctly classify pairwise preferences. Examples include RankSVM~\cite{joachims2006training} and AdaRank~\cite{Jun+Hang:2007} which learn margin classifiers, RankNet~\cite{burges2005learning} which optimizes a probabilistic loss function, and the P-Norm Push method~\cite{Rudin:JMLR:2009} which extends the problem to settings where we mostly care about the top of the ranked list. While the so-called ``pairwise'' methods typically optimize convex upper-bounds of the misranking error, direct optimization methods based on mathematical programming have also been proposed~\cite{RuWa:aistats:2018} albeit for linear hypothesis spaces.

Pairwise learning-to-rank methods, while generally effective, optimize loss functions that are misaligned with more complex ranking statistics such as Expected Reciprocal Rank~\cite{Chapelle:ERR:2009} or NDCG~\cite{jarvelin2002cumulated}. This discrepancy has given rise to the so-called ``listwise'' methods, where the loss function under optimization is defined over the entire list of items, not just pairs.

Listwise learning-to-rank methods either derive a smooth approximation to ranking metrics or use heuristics to construct smooth surrogate loss functions. Algorithms that represent the first class are SoftRank~\cite{Taylor+al:2008} which takes every score to be the mean of a Gaussian distribution, and ApproxNDCG~\cite{qin2010general} which approximates the indicator function---used in the computation of ranks given scores---with a generalized sigmoid.

The other class of listwise learning-to-rank methods include  ListMLE~\cite{xia2008listwise}, ListNet~\cite{cao2007learning}, and $\lambdamart$~\cite{wu2010adapting,burges2010ranknet}. ListMLE maximizes the log-likelihood based on the Plackett-Luce probabilistic model, a loss function that is disconnected from ranking metrics. ListNet minimizes the cross-entropy between the ground-truth and score distributions. Though a recent work~\cite{BruchAnalysisICTIR2019} establishes a link between the ListNet loss function and NDCG under strict conditions---requiring binary relevance labels---in a general setting, its loss is only loosely related to ranking metrics.

$\lambdamart$ is a gradient boosting machine~\cite{friedman2001greedy} that forgoes the loss function altogether and, instead, directly designs the gradients of its unknown loss function using heuristics. While a recent work~\cite{WangLambdaLoss} claims to have found $\lambdamart$'s loss function, it overlooks an important detail: The reported loss function in~\cite{WangLambdaLoss} is not differentiable.

There is abundant evidence to suggest listwise methods are empirically superior to pairwise methods where MRR, ERR, or NDCG determines ranking quality~\cite{WangLambdaLoss,BruchApproxSIGIR2019,liu2009learning}. However, unlike pairwise methods, listwise algorithms remain theoretically ill-understood. Past studies have examined the generalization error bounds for existing surrogate loss functions~\cite{Tewari:ICML2015,Chapelle:IRJ:2010,Lan:ICML:2009}, but little attention has been paid to the validity of such functions which could shed light on their empirical performance.

%% file: preliminaries.tex
\section{Preliminaries}\label{sec:preliminaries}
In this section, we formalize the problem and introduce our notation. To simplify exposition, we write vectors in bold and use subscripts to index their elements (e.g., $\gamma_i \in \bm{\gamma}$).

Let $(\bm{x}, \bm{y}) \in \mathcal{X}^m \times \mathcal{Y}^m$ be a training example comprising of $m$ items and relevance labels where $\mathcal{X} \subset \mathbb{R}^{d}$ is the bounded space of items or item-context pairs represented by $d$-dimensional feature vectors, and $\mathcal{Y} \subset \mathbb{R}_+$ is the space of nonnegative relevance labels. For consistency with existing work on listwise learning-to-rank, we refer to each $x_i \in \bm{x},\,1\leq i \leq m$ as a ``document.'' Note, however, that $x_i$ could be the representation of any general item or item-context pair. We assume the training set $\Psi$ consists of $n$ such examples.

We denote a learning-to-rank scoring function by $f: \mathcal{X} \rightarrow \mathbb{R}$ and assume $f \in \mathcal{F}$ where $\mathcal{F}$ is a compact hypothesis space of bounded functions endowed with the uniform norm. For brevity, we denote $f(x_i)$ by $f_i$ and, with a slight abuse of notation, define $f(\bm{x}) = (f_1, f_2, \ldots, f_m)$, the vector of scores for $m$ documents in $\bm{x}$.

As noted in earlier sections, the goal is to learn a scoring function $f$ that minimizes the empirical risk:
\begin{equation}
\mathcal{L}(f) = \frac{1}{|\Psi|}\sum_{(\bm{x}, \bm{y})\in\Psi} \ell(\bm{y},f(\bm{x})),
\label{equ:empirical_loss}
\end{equation}
where $\ell(\cdot)$ is by assumption a smooth loss function.

\textbf{ListNet}: The loss $\ell$ in ListNet~\cite{cao2007learning} first projects labels $\bm{y}$ and scores $f(\bm{x})$ onto the probability simplex to form distributions $\phi_\text{ListNet}$ and $\rho_\text{ListNet}$, respectively. Given the two distributions, the loss is their distance as measured by cross entropy:
\begin{equation}
    \ell(\bm{y}, f(\bm{x})) \triangleq - \sum_{i=1}^{m}{\phi_\text{ListNet}(y_i) \log \rho_\text{ListNet}(f_i)}.
    \label{equ:cross_entropy}
\end{equation}
The distributions $\phi_\text{ListNet}$ and $\rho_\text{ListNet}$ may be understood as encoding the likelihood of document $x_i$ appearing at the top of the ranked list, referred to as ``top one'' probability, according to the labels and scores respectively. In the original publication~\cite{cao2007learning}, $\phi_\text{ListNet}$ and $\rho_\text{ListNet}$ are defined as follows:
\begin{equation}
    \phi_\text{ListNet}(y_i) = \frac{e^{y_i}}{ \sum_{j=1}^{m}{e^{y_j}} }, \quad
    \rho_\text{ListNet}(f_i) = \frac{e^{f_i}}{ \sum_{j=1}^{m}{e^{f_j}} }.
    \label{equ:listnet_loss}
\end{equation}

\textbf{$\lambdamart$}: The loss $\ell$ in $\lambdamart$ is unknown but its gradients with respect to the scoring function are designed as follows:
\begin{equation}
    \frac{\partial \ell}{\partial f_i} =
      \sum_{y_i > y_j}{\frac{\partial \ell_{ij}}{\partial f_i}} +
      \sum_{y_k > y_i}{\frac{\partial \ell_{ki}}{\partial f_i}},
\label{equ:lambdamart_loss_gradient}
\end{equation}
where 
\begin{equation}
    \frac{\partial \ell_{mn}}{\partial f_m} = \frac{-\sigma | \Delta_{\text{NDCG}_{mn}} |}{1 + e^{\sigma (f_m - f_n)}} = 
    -\frac{\partial \ell_{nm}}{\partial f_m},
\label{equ:lambdamart_loss_gradient_pairs}
\end{equation}
where $\sigma$ is a hyperparameter and $\Delta_{\text{NDCG}_{mn}}$ is the change in NDCG if documents at ranks $m$ and $n$ are swapped. Finally, NDCG is defined as follows:
\begin{equation}
    \text{NDCG}(\bm{\pi}_f, \bm{y}) = \frac{\text{DCG}(\bm{\pi}_f, \bm{y})}{\text{DCG}(\bm{\pi}_{\bm{y}}, \bm{y})},
    \label{equ:ndcg}
\end{equation}
where $\bm{\pi}_f$ is a ranked list induced by $f$ on $\bm{x}$, $\bm{\pi}_{\bm{y}}$ is the ideal ranked list (where $\bm{x}$ is sorted by $\bm{y}$), and $\text{DCG}$ is defined as follows:
\begin{equation}
    \text{DCG}(\bm{\pi}, \bm{y}) =\sum_{i=1}^m \frac{2^{y_i}-1}{\log_2(1 + \bm{\pi}[i])},
    \label{equ:dcg}
\end{equation}
with $\bm{\pi}[i]$ denoting the rank of $x_i$.

%% file: proposed_method.tex
\section{Proposed Method}\label{sec:proposed_method}
In this section, we show how a slight modification to the ListNet loss function equips the loss with interesting theoretical properties. To avoid conflating implementation details with the loss function itself, we name our proposed loss function $\xendcg$.

\begin{definition}
For a training example $(\bm{x}, \bm{y}) \in \mathcal{X}^m \times \mathcal{Y}^m$ and scores $f(\bm{x}) \in \mathbb{R}^m$, we define $\xendcg$ as the cross entropy between score distribution $\rho$ and a parameterized class of label distributions $\phi$ defined as follows:
\begin{equation*}
    \rho(f_i) = \frac{e^{f_i}}{ \sum_{j=1}^{m}{e^{f_j}} },\quad
    \phi(y_i;\, \bm{\gamma}) = \frac{2^{y_i} - \gamma_i}{\sum_{j=1}^m 2^{y_j} - \gamma_j}
\end{equation*}
where $\bm{\gamma} \in [0,\,1]^m$.
\label{def:xe_ndcg}
\end{definition}

In effect, the distribution $\phi$ allocates a mass in the interval $[2^{y_r} - 1,\, 2^{y_r}]$ for each document indexed with $r$. As we will explain later, the vector $\bm{\gamma}$ plays an important role in certain theoretical properties of our proposed loss function. Note that in general, $\bm{\gamma}$ may be unique to each training example $(\bm{x},\,\bm{y})$.

\subsection{Relationship to NDCG}
The difference between $\xendcg$ and ListNet is minor but consequential: The change to the definition of $\phi$ leads to our main result.

\begin{theorem}
$\xendcg$ is an upper-bound on negative (translated and log-transformed) mean Normalized Discounted Cumulative Gain.\label{thm:xe_ndcg_bound}
\end{theorem}

Theorem~\ref{thm:xe_ndcg_bound} asserts that $\xendcg$ is a convex proxy to minimizing negative NDCG (where we turn NDCG which is a utility to a cost by negation). No such analytical link exists between the $\lambdamart$, ListNet, or other listwise learning-to-rank loss functions and ranking metrics.

In proving Theorem~\ref{thm:xe_ndcg_bound} we make use of Jensen's inequality when applied to the $\log$ function:
\begin{equation}
\log \mathbb{E}[X] \geq \mathbb{E}[\log X],
\label{equ:jensens}
\end{equation}
where $X$ is a random variable and $\mathbb{E}[\cdot]$ denotes expectation. We also use the following bound on ranks that was originally derived in~\cite{BruchAnalysisICTIR2019}:
\begin{align*}
    \bm{\pi}[r] &= 1 + \sum_{i \neq r} \mathbbm{1}_{f_i > f_r}
         = 1 + \sum_{i \neq r} \mathbbm{1}_{f_i - f_r > 0} \\
     &\leq 1 + \sum_{i \neq r} e^{(f_i - f_r)}
         = \sum_{i} e^{(f_i - f_r)}
         = \frac{\sum_{i} e^{f_i}}{e^{f_r}},
\end{align*}
where $\mathbbm{1}_p$ is the indicator function taking the value $1$ when the predicate $p$ is true and $0$ otherwise. The above leads to:
\begin{equation}
    \frac{1}{\bm{\pi}[r]} \geq \frac{e^{f_r}}{\sum_{i} e^{f_i}} = \rho(f_r).
    \label{equ:reciprocal_rank}
\end{equation}
\begin{proof}
Consider DCG$(\bm{\pi}_{\bm{y}}, \bm{y})$. Using $\log_2(1 + z) \geq 1,\, \forall z \geq 1$:
\begin{align}
    \text{DCG}&(\bm{\pi}_{\bm{y}}, \bm{y}) = \sum_{i=1}^{m} \frac{2^{y_i} - 1}{\log_2(1 + \bm{\pi}_{\bm{y}}[i])}
    \leq \sum_{i=1}^m 2^{y_i} - \gamma_i,
    \label{equ:inverse_dcg_max_inequality}
\end{align}
for $0 \leq \gamma_i \leq 1$.

Turning to DCG$(\bm{\pi}_f, \bm{y})$ and using $1 + z \leq 2^z$ for a positive integer $z$ or equivalently $\log_2(1+z) \leq z$, we have the following:
\begin{align}
    &\text{DCG}(\bm{\pi}_f, \bm{y}) = \sum_{r} \frac{2^{y_r} - 1}{\log_2(1 + \bm{\pi}_f[r])}
    \geq \sum_{r} \frac{2^{y_r} - 1}{\bm{\pi}_f[r]} \nonumber \\
    &\geq \sum_{r} (2^{y_r} - 1) \rho(f_r)
    = \big[ \sum_{r} 2^{y_r} \rho(f_r) \big] - 1 \nonumber \\
    &\geq \big[ \sum_r (2^{y_r} - \gamma_r) \rho(f_r) \big] - 1,
    \label{equ:dcg_inequality}
\end{align}
where the second inequality holds by Equation~(\ref{equ:reciprocal_rank}).

Finally, consider a translation (by a constant) and $\log$-transformation of mean NDCG, $\overline{\text{NDCG}}$, as follows:
\begin{equation*}
\widetilde{\text{NDCG}} \triangleq \log \big( \overline{\text{NDCG}} + \frac{1}{\vert \Psi \vert} \sum_{(\bm{x}, \bm{y})} \frac{1}{\text{DCG}(\bm{\pi}_{\bm{y}}, \bm{y})} \big).
\end{equation*}
Given the monotonicity of $\log(\cdot)$, the maximizer of $\widetilde{\text{NDCG}}$ also maximizes $\overline{\text{NDCG}}$. We now proceed as follows:
\begin{align}
    \widetilde{\text{NDCG}} &= \log \frac{1}{\vert \Psi \vert} \sum_{(\bm{x}, \bm{y})}{ \frac{1}{\text{DCG}(\bm{\pi}_{\bm{y}}, \bm{y})} \big[ \text{DCG}(\bm{\pi}_f, \bm{y}) + 1 \big] } \nonumber \\
    &\geq \log \frac{1}{\vert \Psi \vert} \sum_{(\bm{x}, \bm{y})}{ \frac{1}{\sum_{j} 2^{y_j} - \gamma_j} \big[ \text{DCG}(\bm{\pi}_f, \bm{y}) + 1\big]} \nonumber \\
    &\geq \log \frac{1}{\vert \Psi \vert} \sum_{(\bm{x}, \bm{y})}{ \sum_{r} \phi(y_r) \rho(f_r)} \label{equ:bound_proof:before_jensens} \\
    &\geq \frac{1}{\vert \Psi \vert} \sum_{(\bm{x}, \bm{y})}{ \sum_{r} \phi(y_r) \log \rho(f_r) }, \label{equ:bound_proof:after_jensens}
\end{align}
where the first inequality holds by Equation~(\ref{equ:inverse_dcg_max_inequality}), the second by Equation~(\ref{equ:dcg_inequality}) and Definition~\ref{def:xe_ndcg}, and the last by repeated applications of Equation~(\ref{equ:jensens}). Negating both sides completes the proof.
\end{proof}

While establishing that $\mathcal{L}$ bounds the NDCG loss is a necessary property in a surrogate, it is not sufficient. As an example, the constant function $\mathcal{L}(f)=2$ bounds the NDCG loss, but optimizing it does not lead to an optimal $f^{\ast}$. This is where the notion of Fisher consistency becomes critical: In summary, a loss function is consistent with an evaluation measure, if the optimal solution to the loss function is also an optimal solution of the evaluation measure.

The cross entropy function that is at the heart of ListNet and our proposed method was shown to be consistent with the $0-1$ ranking loss in~\cite{xia2008listwise}. In general, however, the loss is not consistent with NDCG~\cite{pmlr-v15-ravikumar11a}. But under special conditions that are common in information retrieval, cross entropy (and as a result $\xendcg$) become NDCG-consistent.

\begin{claim}
$\xendcg$ is NDCG-consistent on datasets with graded relevance judgments or with a single click per query.\label{thm:xe_ndcg_consistency}
\end{claim}
\begin{proof}
We omit a complete proof due to space constraints, but note that the above is a trivial consequence of the conditions. Briefly:~\cite{pmlr-v15-ravikumar11a} shows that $\mathcal{L}$ is NDCG-consistent so long as its terms, $\ell(\bm{y}, f)$, are each normalized by the best DCG, $\text{DCG}(\bm{\pi}_{\bm{y}}, \bm{y})$. In settings where queries receive a single click, the best DCG is simply 1, and so a $0-1$ loss-consistent surrogate is naturally NDCG-consistent too. Furthermore, when queries do not repeat as in datasets with graded relevance labels, every query has a unique relevance vector. That degenerate relevance probability distribution renders the expectation and thus normalization terms irrelevant, thereby equipping a $0-1$ loss-consistent surrogate with NDCG-consistency.
\end{proof}

\subsection{Comparison with $\lambdamart$}\label{sec:proposed_method:complexity}
In this section, we compare $\xendcg$ with $\lambdamart$ in terms of model complexity and generalization error. In what follows, we proceed under the strong assumption that the loss optimized by $\lambdamart$ in fact exists. That is, we assume that there exists a differentiable function that satisfies Equation~(\ref{equ:lambdamart_loss_gradient}).

We begin with an examination of the Lipschitz constant of the two algorithms---an upper-bound on the variation a function can exhibit. Intuitively, functions with a smaller Lipschitz constant are simpler and thus generalize better.

\begin{claim}
The $\lambdamart$ loss is $\sigma m^2$-Lipschitz with respect to $\|\cdot\|_{\infty}$.\label{prop:lambdamart_lipschitz}
\end{claim}
\begin{proof}
Recall the definition of the Lipschitz constant for a differentiable function $h(\cdot)$:
\begin{align*}
    \mathit{Lip}_h &= \sup \frac{
        | h(f) - h(f^\prime)|}{\| f - f^\prime \|} \\
        &= \sup \frac{
        | \nabla_{f} h(f^{\prime\prime}) (f - f^\prime) |}{\| f - f^\prime \|}
        = \| \nabla_{f} h \|_{\ast},
\end{align*}
where the second equality holds by the Mean Value Theorem and the last by the definition of the dual norm, $\| \cdot \|_{\ast}$. Therefore, to derive the Lipschitz constant of a function with respect to the infinity norm, it is sufficient to calculate the $L_1$ norm of its gradient. Given that $\lambdamart$'s loss function is unknown, we resort to this strategy to derive its Lipschitz constant.

Observe that the terms in Equation~(\ref{equ:lambdamart_loss_gradient_pairs}) are bounded by $\sigma$ and Equation~(\ref{equ:lambdamart_loss_gradient}) has at most $m$ such terms. As such, we have that,
\begin{equation*}
    | \frac{\partial \ell}{\partial f_i} | \leq \sigma m.
\end{equation*}
Then,
\begin{equation*}
    \| \nabla_{f} \ell \|_{1} = \sum_{i=1}^{m} {| \frac{\partial \ell}{\partial f_i} |} \leq \sum_{i=1}^{m} {\sigma m} = \sigma m^2
\end{equation*}
which completes the proof.
\end{proof}

\begin{proposition}
$\xendcg$ is $2$-Lipschitz with respect to $\|\cdot\|_{\infty}$.\label{prop:xe_ndcg_lipschitz}
\end{proposition}
\begin{proof}
Recall that the cost function $\ell(\cdot)$ for $\xendcg$ is defined as follows:
\begin{equation*}
    \ell(\bm{y}, f(\bm{x})) \triangleq -\sum \phi(y_i) \log \rho(f_i),
\end{equation*}
where $\phi$ and $\rho$ form probability distributions over labels $\bm{y}$ and scores $f(\bm{x})$ respectively, and $f_i = f(x_i)$.

Observe that the derivative of the cost function $\ell$ with respect to a score $f_r$ is:
\begin{align*}
    \frac{\partial \ell}{\partial f_r} &= \frac{\partial}{\partial f_r} [-\sum_i \phi(y_i) (f_i - \log \sum_j e^{f_j} )] \\
    &=  \frac{\partial}{\partial f_r} [(\sum_i -\phi(y_i) f_i) + \log \sum_j e^{f_j}] \\
    &= -\phi(y_r) + \frac{e^{f_r}}{\sum_j e^{f_j}}
    = -\phi(y_r) + \rho(f_r).
\end{align*}
We then have that,
\begin{equation*}
    | \frac{\partial \ell}{\partial f_r} | \leq \phi(y_r) + \rho(f_r),
\end{equation*}
resulting in,
\begin{equation*}
    \| \nabla_{f} \ell \|_{1} = \sum {| \frac{\partial \ell}{\partial f_i} |} \leq \sum (\phi(y_r) + \rho(f_r)) = 2
\end{equation*}
as required.
\end{proof}

In order to put this difference into perspective, we use the results above to derive bounds on the generalization error of the two algorithms. But first we need the following result.

\begin{theorem} Let $\mathcal{F}$ be a compact space of bounded functions from $\mathcal{X}$ to $[0,1]$, $n=|\Psi|$ be the number of training examples, $\mathit{Lip}_\ell$ the Lipschitz constant of loss function $\ell$, and $\mathfrak{N}(\frac{\epsilon}{4\mathit{Lip}_\ell}, \mathcal{F}, \|\cdot\|_\infty)$ the covering number of $\mathcal{F}$ by $L_\infty$ balls of radius $\frac{\epsilon}{4\mathit{Lip}_\ell}$. The following generalization error bound holds:
\begin{equation*}
\mathcal{P}\{ \mathcal{E}(f) \leq \epsilon \} \geq 1 - 2\mathfrak{N}(\frac{\epsilon}{4\mathit{Lip}_\ell}, \mathcal{F}, \|\cdot\|_\infty)\mathit{exp}(\displaystyle\frac{-2n\epsilon^2}{\mathit{Lip}_\ell^2}),
\end{equation*}
where the generalization error $\mathcal{E}$ is defined as follows:
\begin{equation*}
    \mathcal{E}(f) \triangleq \displaystyle\mathop{\mathbb{E}}_{\mathcal{X}^m\times\mathcal{Y}^m} [\ell(\bm{y}, f(\bm{x}))] - \frac{1}{n} \sum_{(\bm{x}, \bm{y}) \in \Psi} {\ell(\bm{y}, f(\bm{x}))}.
\end{equation*}\label{thm:generalization_error_bound}
\end{theorem}
\begin{proof}
Based on the proofs in~\cite{Cucker02onthe,Rudin:JMLR:2009}.
\end{proof}

The dependence of the generalization error bound on the Lipschitz constant suggests that unlike $\lambdamart$, $\xendcg$'s generalization error does not degrade as the number of documents per training example increases. Furthermore, given its larger Lipschitz constant and potentially higher complexity, we hypothesize that $\lambdamart$ is less robust to noise and generalizes poorly in settings where the number of documents per training example is large.

We note that, the independence of the ListNet generalization error bound from $m$ was also reported in~\cite{Tewari:ICML2015} for linear models, but we present the (structure of the) bounds here to allow a direct comparison between $\lambdamart$ and $\xendcg$.

We conclude this section with the following note: It is true that the Lipschitz constant is only a loose measure of the complexity of a function. We naturally do not expect the bounds to hold exactly in practice, but we expect the bounds to hint at an algorithm's behavior in extremes. As our experiments show later, an empirical comparison of the two functions is in alignment with the analysis above: As the experimental setup approaches more extreme levels of noise, a likely scenario in click data, the two algorithms behave very differently.

\subsection{Approximating the Inverse Hessian}
In this work, we fix the hypothesis space, $\mathcal{F}$, to Gradient Boosted Regression Trees. This is, in part, because we are interested in a fair comparison of ListNet, $\xendcg$, and $\lambdamart$ in isolation of other factors, as explained in Section~\ref{sec:experiments}. As most GBRT learning algorithms use second-order optimization methods (e.g., Newton's), however, we must approximate the inverse Hessian for ListNet and $\xendcg$.

Unfortunately, $\xendcg$ as defined in Definition~\ref{def:xe_ndcg} results in a Hessian that is singular, making the loss incompatible with a straightforward implementation of Newton's second-order method. We resolve this technical difficulty by making a small adjustment to the formulation of the loss function.

Let us re-define the score distribution,
$\rho$, from Definition~\ref{def:xe_ndcg} as follows for a negligible $\epsilon>0$:
\begin{equation}
    \rho(f_i) = \frac{e^{f_i}}{ \sum_{j=1}^{m}{e^{f_j} + \epsilon} }.
\end{equation}
In effect, we take away a small probability mass, $\rho(f_{m+1}) = \epsilon/(\sum e^{f_j} + \epsilon)$, from the score distribution for a nonexistent, $m+1^{\textup{th}}$ document with label probability $\phi(f_{m+1}) = 0$. The gradients of the loss will take the following form:
\begin{align*}
    \frac{\partial \ell}{\partial f_r} &= \frac{\partial}{\partial f_r} [\sum_i (-\phi(y_i) f_i) + \log (\sum_j e^{f_j} + \epsilon)] \\
    & = -\phi_r + \rho_r,
\end{align*}
where $\phi_r = \phi(y_r)$ and $\rho_r = \rho(f_r)$. The Hessian looks as follows:
\begin{equation*}
    H_{ij} =
    \begin{cases}
      \rho_i (1 - \rho_i), & i = j \\
      -\rho_i \rho_j, & i \neq j
    \end{cases}
\end{equation*}

\begin{claim}
The Hessian, as defined above, is positive definite.\label{claim:hessian_nonsingular}
\end{claim}
\begin{proof}
A complete proof may be found in the appendix. Observe that $H$ is strictly diagonally dominant:
\begin{align*}
    |H_{kk}| &= \rho_k(1- \rho_k) = \rho_k (1 - \frac{e^{f_k}}{\sum e^{f_j} + \epsilon}) \\
    &= \rho_k \frac{ \sum_{j \neq k} e^{f_j} + \epsilon }{ \sum e^{f_j} + \epsilon }
      > \rho_k \sum_{j \neq k} \rho_j = \sum_{j \neq k} |H_{kj}|.
\end{align*}

By the properties of strictly diagonally dominant matrices and the fact that the diagonal elements of $H$ are positive, we have that $H \succ 0$ and therefore invertible.
\end{proof}

We now turn to approximating the inverse of $H$ as required. Write $H = D (I - S)$ where $I$ is the identity matrix, $D$ is a diagonal matrix where $D_{ii} = \rho_i (1 - \rho_i)$
and $S$ is a square matrix where,
\begin{equation*}
    S_{ij} =
    \begin{cases}
    0, & i=j\\
    \rho_j/(1 - \rho_i), & i \neq j
    \end{cases}.
\end{equation*}

\begin{claim}
The spectral radius of $S$ is strictly less than 1.\label{claim:spectral_radius}
\end{claim}
\begin{proof}
A complete proof is presented in the appendix. $S$ is a square matrix with nonnegative entries. By the Perron-Frobenious theorem, its spectral radius is bounded above by the maximum row-wise sum of entries, which, in $S$, is strictly less than 1.
\end{proof}

Claim~\ref{claim:spectral_radius} allows us to apply Neumann's result to approximate $(I - S)^{-1}$ as follows:
\begin{equation*}
    (I - S)^{-1} = \sum_{k = 0}^{\infty} S^k \approx I + S + S^2.
\end{equation*}
Using this result, we may approximate $H^{-1}$ as follows:
\begin{equation*}
    H^{-1} = (I - S)^{-1} D^{-1} \approx (I + S + S^2) D^{-1}
\end{equation*}

With that, we can finally calculate the update rule in Newton's method which requires the quantity $H^{-1}\nabla$:
\begin{align*}
    &(H^{-1}\nabla)_k = \sum_i H^{-1}_{ki} \nabla_i \\
    &\approx \sum_i (I + S + S^2)_{ki} (D^{-1} \nabla)_i \\
    &= \sum_i (I + S + S^2)_{ki} \frac{-\phi_i + \rho_i}{\rho_i (1 - \rho_i)} \\
    &= \underbrace{ \frac{-\phi_k + \rho_k}{\rho_k (1 - \rho_k)} }_{(ID^{-1}\nabla)_k} + \underbrace{ \frac{1}{1-\rho_k} \sum_{i \neq k} \frac{-\phi_i + \rho_i}{1 - \rho_i} }_{(SD^{-1}\nabla)_k} + \sum_{i \neq k}  \frac{\rho_i (SD^{-1}\nabla)_i}{1 - \rho_k}\\
    &= \frac{-\phi_k + \rho_k + \rho_k \sum_{i \neq k} \frac{-\phi_i + \rho_i}{1 - \rho_i} + \rho_k \sum_{i \neq k}  \rho_i (SD^{-1}\nabla)_i}{\rho_k(1 - \rho_k)}.
\end{align*}

%% file: experiments.tex
\section{Experiments}\label{sec:experiments}
We are largely interested in a comparison of (a) the overall performance of ListNet, $\lambdamart$, and $\xendcg$ on benchmark learning-to-rank datasets, and (b) their robustness to various types and degrees of noise as a proxy to complexity. In this section, we describe our experimental setup and report our empirical findings.

It is important to note that there is an extensive list of published work~\cite{BruchApproxSIGIR2019,Zhuang:arxiv:2005.02553,Pasumarthi:arxiv:1910.09676,Bruch:wsdm:2020} that compare learning-to-rank algorithms on benchmark datasets we use in this work. We rely on prior research and do not include methods that have been shown repeatedly to be weaker than $\lambdamart$, including BoltzRank~\cite{volkovs:icml:2009}, ListMLE~\cite{xia2008listwise}, Position-Aware ListMLE~\cite{Lan:uai:2014}, SoftRank~\cite{Taylor+al:2008}, ApproxNDCG~\cite{BruchApproxSIGIR2019,qin2010general}, or other direct optimization methods~\cite{metzler2005directMaximization,Xu2008DirectlyOptimizingLTR}.

\subsection{Datasets}\label{sec:experiments:datasets}
We conduct experiments on two publicly available benchmark datasets: MSLR Web30K~\cite{DBLP:journals/corr/QinL13} and Yahoo! Learning to Rank Challenge Set 1~\cite{chapelle2011yahoo}. Web30K contains roughly 30,000 examples, with an average of 120 documents per example. Documents are represented by 136 numeric features. Yahoo! also has about the same number of examples but with an average of 24 documents per example, each represented by 519 features. Documents in both datasets are labeled with graded relevance from 0 to 4 with larger labels indicating a higher relevance.

From each dataset, we sample training (60\%), validation (20\%), and test (20\%) examples, and train and compare models on the resulting splits. We repeat this procedure 100 times and obtain mean NDCG at different rank cutoffs for each trial. We subsequently compare the ranking quality between pairs of models and determine statistical significance of differences using a paired \emph{t}-test.

During evaluation, we discard examples with no relevant documents. There are 982 and 1,135 such examples in the Web30K and Yahoo! datasets. The reason for ignoring these examples during evaluation is that their ranking quality can be arbitrarily 0 or 1, which only skews the average.

\begin{table}[t]
\caption{NDCG (percentage) on test sets at rank cutoffs 5 and 10, averaged over 100 randomized trials. In each trial, training, validation, and test sets are sampled from the datasets. The differences at all rank cutoffs between all models are statistically significant according to a paired \emph{t}-test ($\alpha=0.01$).}
\label{table:model_comparison}
\vskip 0.15in
\begin{center}
\begin{sc}
\begin{tabular}{lcc|cc}
& \multicolumn{2}{c}{Web30K} & \multicolumn{2}{c}{Yahoo!} \\
\toprule
Model & @5 & @10 & @5 & @10 \\
\midrule
ListNet & 47.68 & 49.76 & 71.76 & 76.52 \\
$\lambdamart$ & 48.08 & 49.94 & 73.00 & 77.49 \\
$\xendcgmart$ & 48.23 & 50.27 & 73.37 & 77.84 \\
\bottomrule
\end{tabular}
\end{sc}
\end{center}
\end{table}

\subsection{Models}
We train $\lambdamart$ models using LightGBM~\cite{lightgbm2017nips}. The hyperparameters are guided by previous work~\cite{lightgbm2017nips, WangLambdaLoss,BruchAnalysisICTIR2019}. For Web30K, max\_bin is 255, learning\_rate is 0.02, num\_leaves is 400, min\_data\_in\_leaf is 50, min\_sum\_hessian\_in\_leaf is set to 0, $\sigma$ is 1, and lambdamart\_norm is set to false. For Yahoo!, num\_leaves is 200 and min\_data\_in\_leaf is 100. We use NDCG@5 to select the best models on validation sets and fix early stopping round to 50 up to 500 trees.

We also implemented ListNet and $\xendcg$ in LightGBM.~\footnote{Available at github.com/microsoft/LightGBM.} As noted earlier, by fixing the hypothesis space to gradient boosted trees, we aim to strictly compare the performance of the loss functions and shield our analysis from any effect the hypothesis space may have on convergence and generalization. An additional reason for choosing gradient boosted trees is that, recent evidence~\cite{BruchApproxSIGIR2019,Zhuang:arxiv:2005.02553,Pasumarthi:arxiv:1910.09676} confirm their superior performance against other hypothesis spaces such as deep neural networks, at least on the benchmark datasets we use in this work. We refer to the tree-based model trained by optimizing $\xendcg$ as $\xendcgmart$. For training purposes, we use the same hyperparameters above.

Finally, we must address the choice for $\bm{\gamma}$ in $\xendcgmart$. In this work, we simplify the choice by sampling $\bm{\gamma}$ uniformly from $[0,\,1]^m$ for every training example (with $m$ documents) and at every iteration of boosting. We leave a detailed examination of the effect of this parameter to a future study.

\subsection{Ranking Quality}
\begin{figure}[t]
\begin{center}
\centerline{
\subfloat[Web30K]{
\includegraphics[height=1.3in]{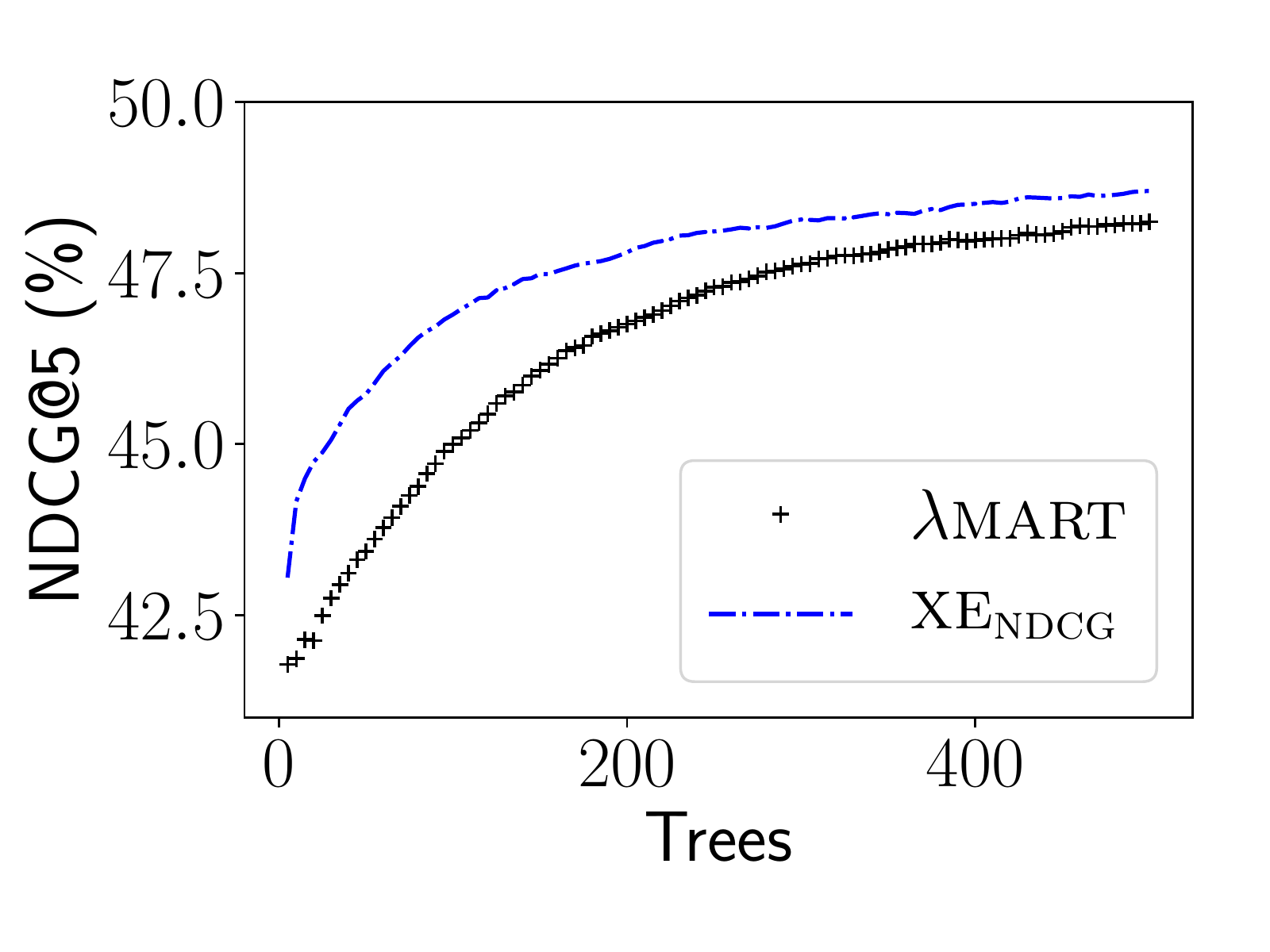}}
\hspace{-1em}
\subfloat[Yahoo!]{
\includegraphics[height=1.3in]{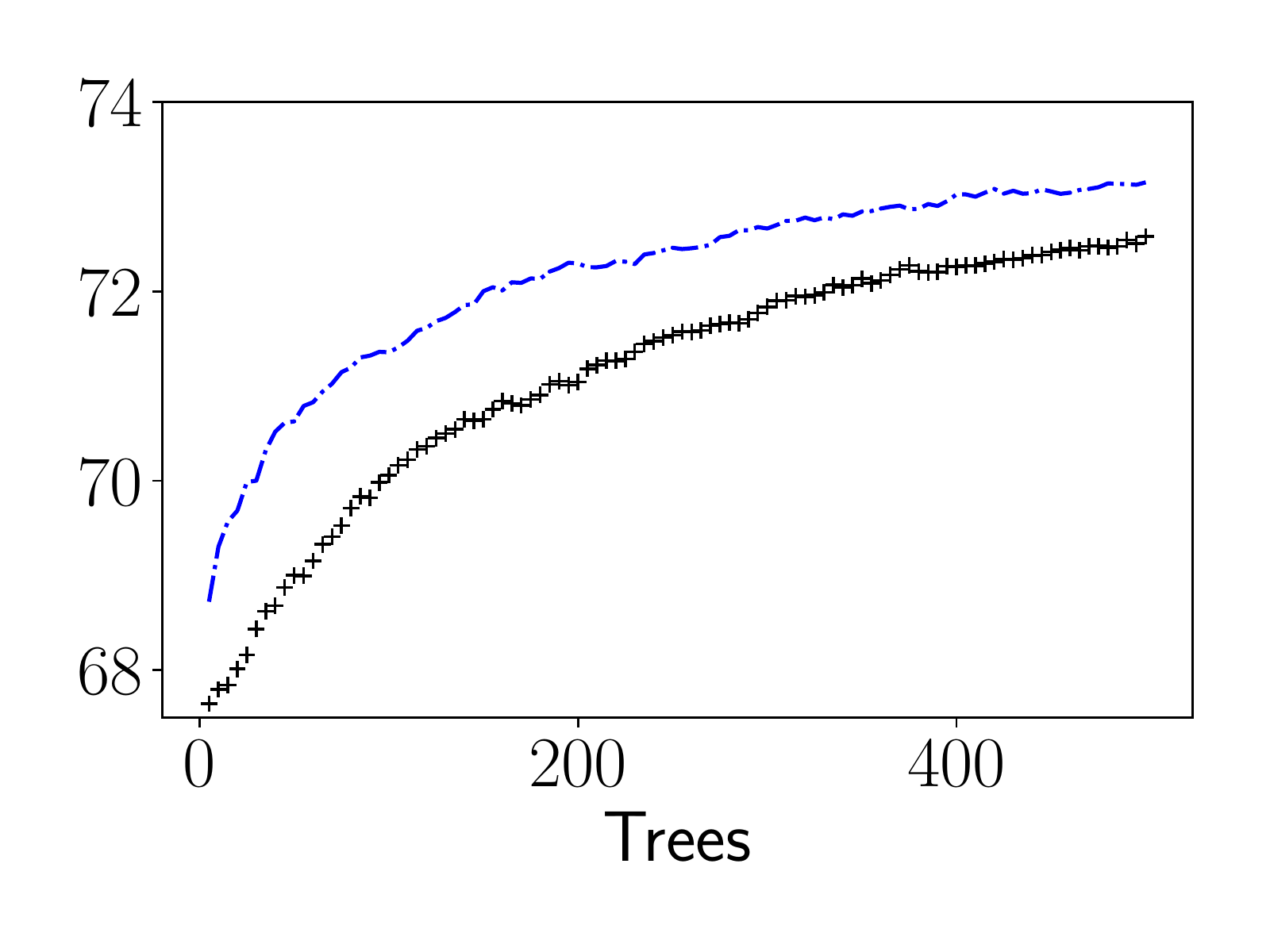}}
}
\caption{NDCG@5 on validation sets during training of $\lambdamart$ and $\xendcgmart$ in a representative trial.}
\label{fig:convergence}
\end{center}
\end{figure}

We compare the ranking quality of the three models under consideration. We report model quality by measuring average NDCG at rank cutoffs 5 and 10. As noted earlier, we also measure statistical significance in the difference between model qualities using a paired \emph{t}-test with significance level set to $\alpha=0.01$. Our results are summarized in Table~\ref{table:model_comparison}.

From Table~\ref{table:model_comparison}, we observe that ListNet consistently performs poorly across both datasets. The quality gap between ListNet and $\lambdamart$ is statistically significant at all rank cutoffs. This observation is in agreement with past studies~\cite{BruchAnalysisICTIR2019}.

On the other hand, our proposed $\xendcgmart$ yields a significant improvement over ListNet. This observation holds consistently across both datasets and rank cutoffs and lends support to our theoretical findings in previous sections.

Not only does $\xendcgmart$ outperform ListNet, its performance surpasses that of $\lambdamart$'s. While $\xendcgmart$'s gain over $\lambdamart$ is smaller than its gap with ListNet, the differences are statistically significant. This is an encouraging result: $\xendcgmart$ is not only theoretically sound and is equipped with better properties, it also performs well empirically compared to the state-of-the-art algorithm.

\begin{figure*}[t]
\begin{center}
\centerline{
\subfloat[Web30K]{
\includegraphics[height=1.4in]{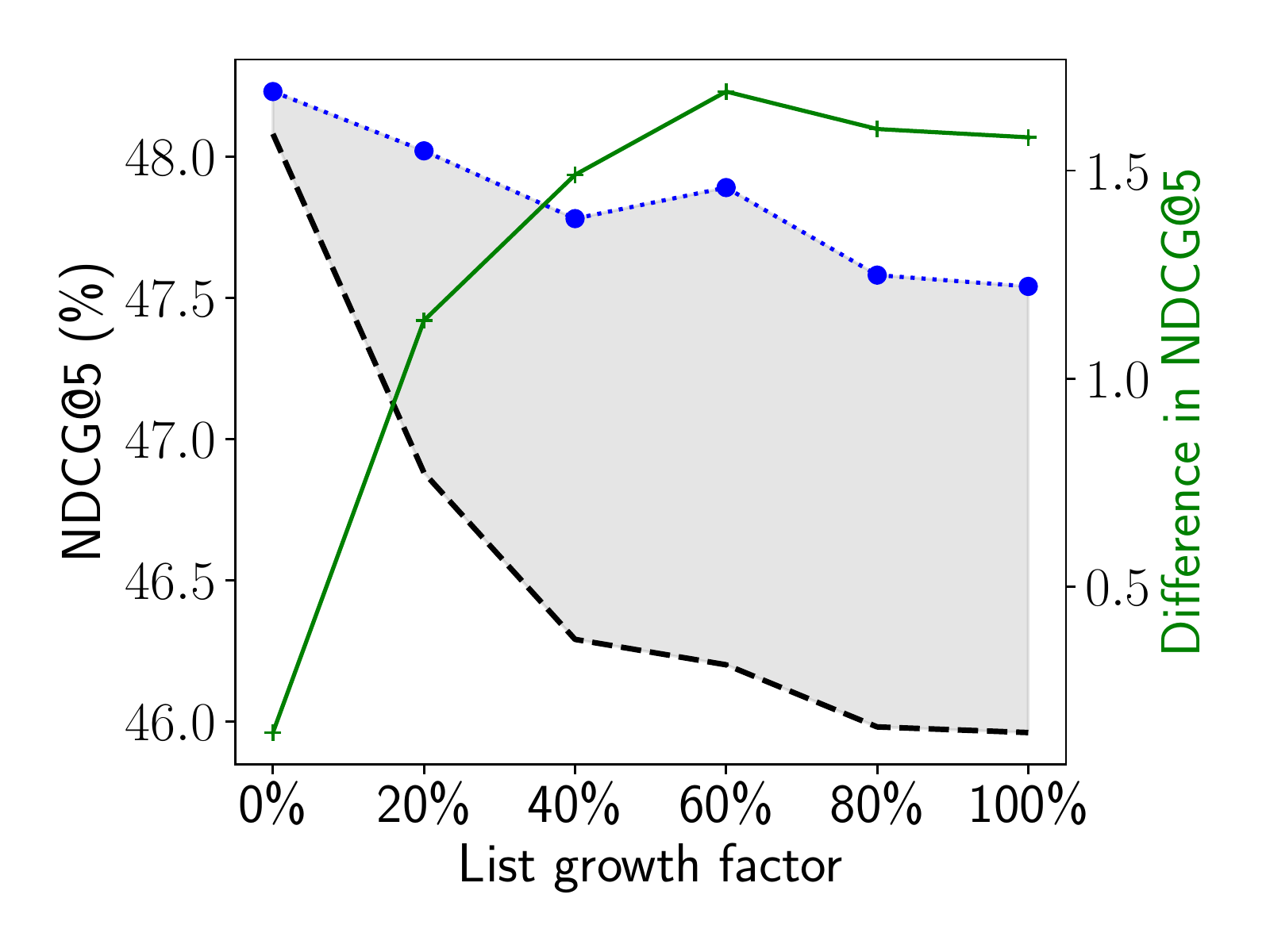}\label{fig:robustness:augmented_list_web30k}}
\hspace{-1em}
\subfloat[Yahoo!]{
\includegraphics[height=1.4in]{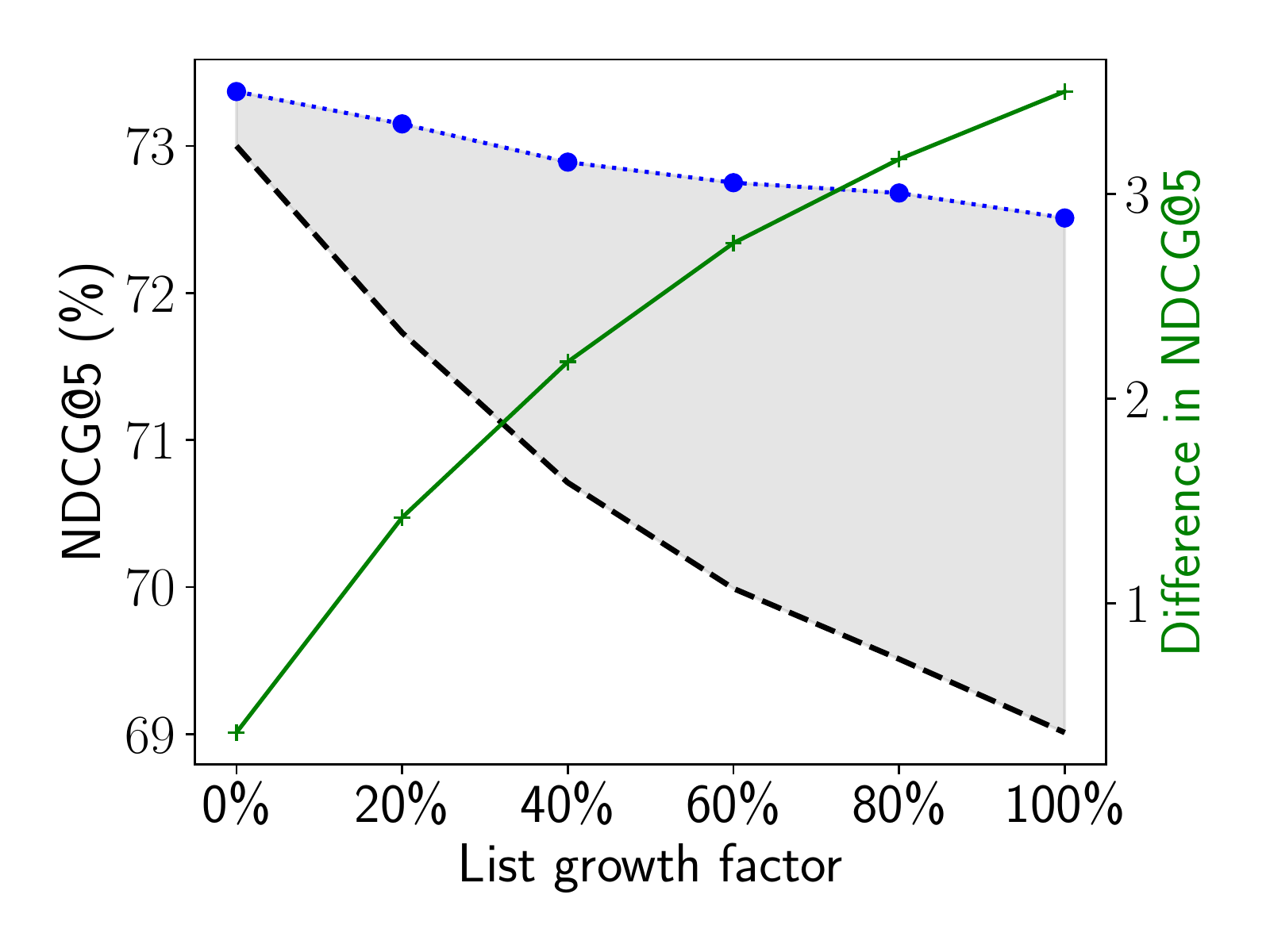}\label{fig:robustness:augmented_list_yahoo}}
\hspace{-1em}
\subfloat[Web30K]{
\includegraphics[height=1.4in]{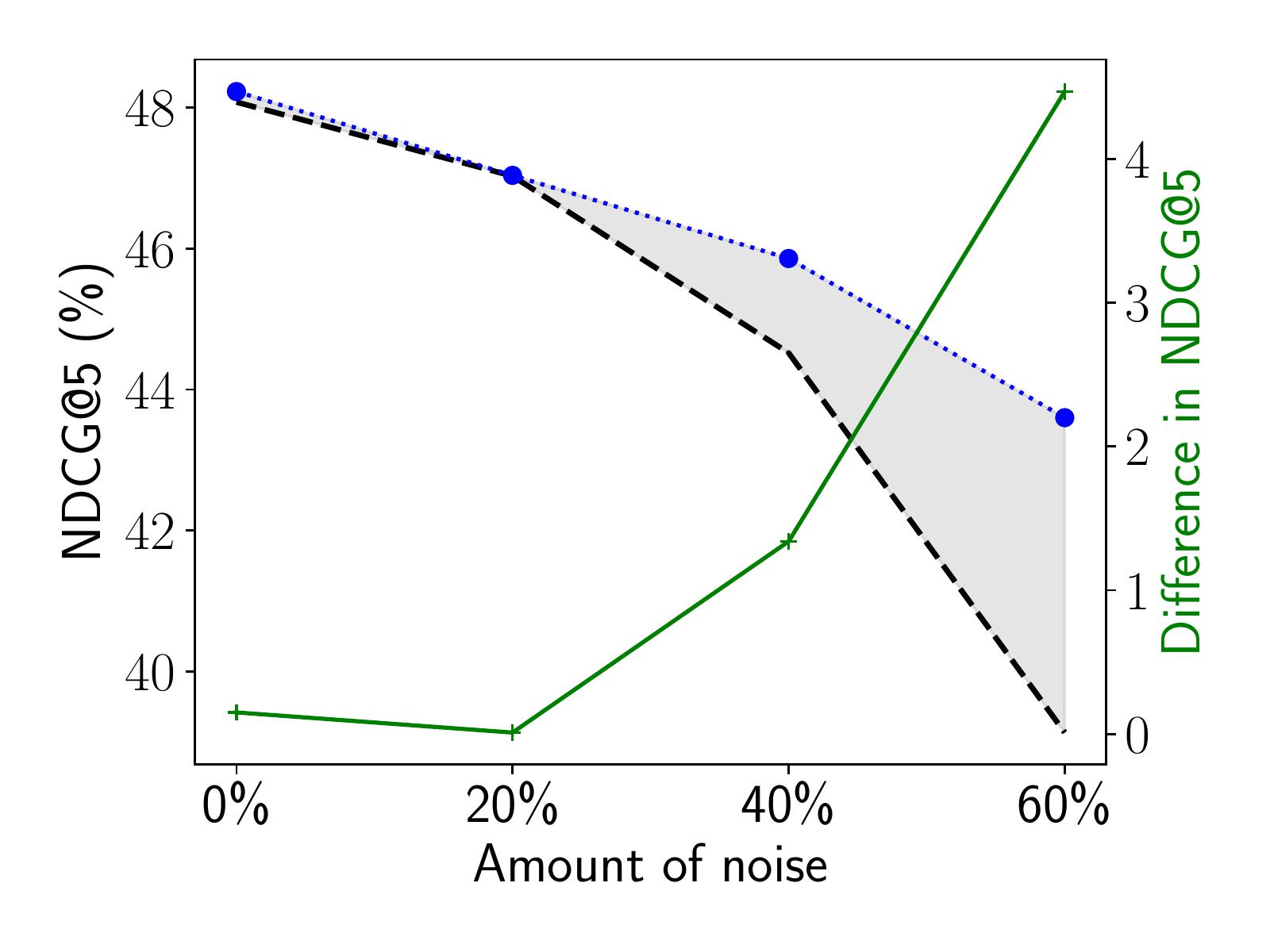}\label{fig:robustness:noise_web30k}}
\hspace{-1em}
\subfloat[Yahoo!]{
\includegraphics[height=1.4in]{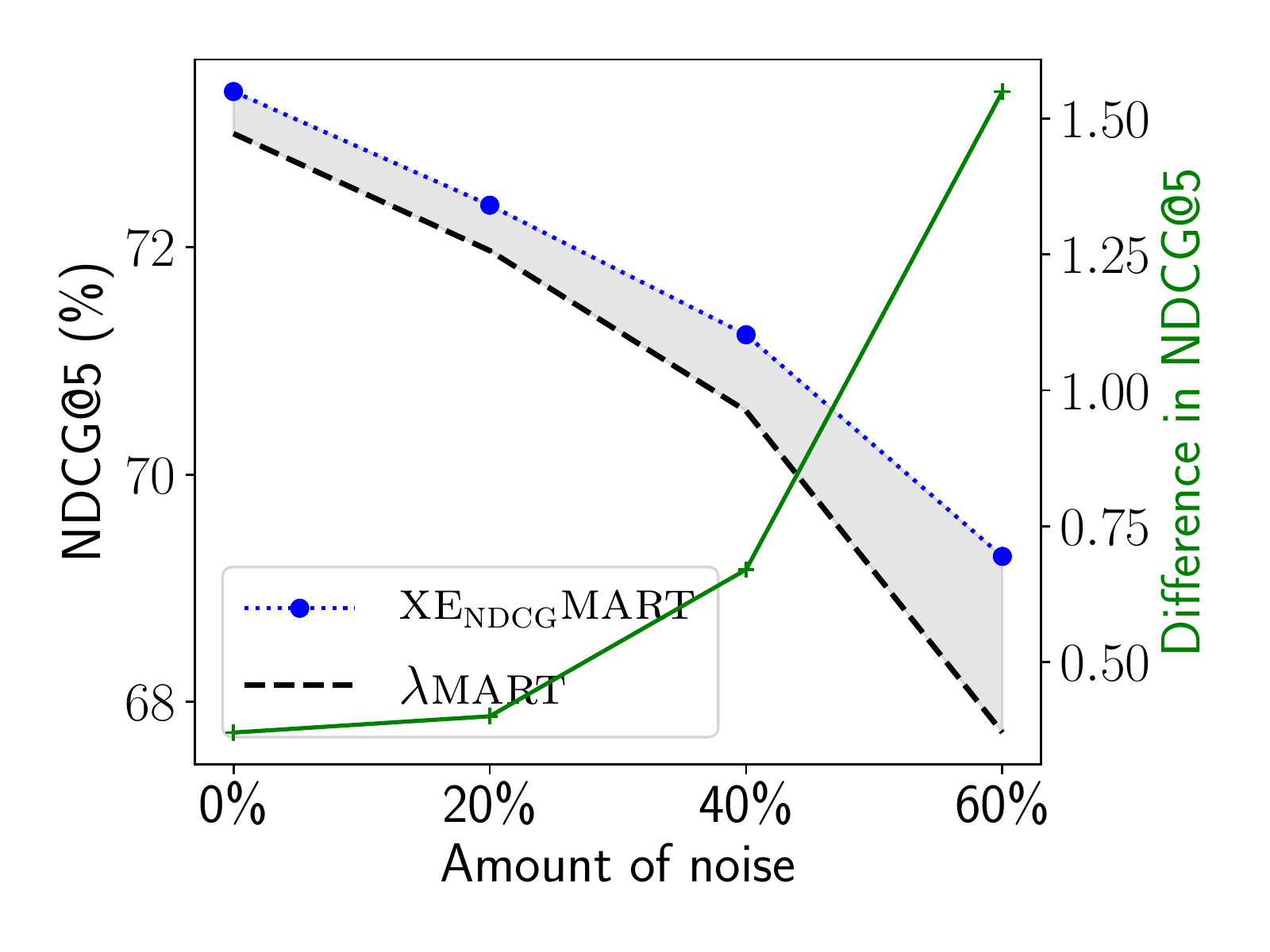}\label{fig:robustness:noise_yahoo}}
}
\caption{Mean NDCG@5 on test sets, averaged over 100 trials. To reduce clutter, legends appear only in the rightmost figure. In each trial, training, validation, and test sets are sampled from the dataset. In (a) and (b), training examples are augmented by additional (randomly sampled) negative documents. For example, the data point at ``40\%'' indicates a 40\% increase in the number of documents for every example. In (c) and (d), a percentage of relevance labels are set to a random value. The solid (green) lines show the difference in NDCG@5 between the two models in absolute terms.}
\label{fig:robustness}
\end{center}
\end{figure*}

A notable difference between $\lambdamart$ and $\xendcgmart$ is in their convergence rate. Figure~\ref{fig:convergence} plots NDCG@5 on validation sets as more trees are added to the ensemble. To avoid clutter, the figure illustrates just one trial (out of 100) but we note that we observe a similar trend across trials. From Figure~\ref{fig:convergence}, it is clear that $\xendcgmart$ outperforms $\lambdamart$ by a wider margin when the number of trees in the ensemble is small. This property is important in latency-sensitive applications where a smaller ensemble is preferred.

\subsection{Robustness to Noise in Graded Labels}\label{sec:experiments:robustness}
We now turn to model robustness where we perform a comparative analysis of the effect of noise on $\lambdamart$ and $\xendcgmart$. The robustness of a ranking model to noise is important in practice due to uncertainty in relevance labels, whether judged by human experts or deduced from user feedback such as clicks. We expect $\lambdamart$ to overfit to noise and be less robust due to its higher model complexity---see findings in Section~\ref{sec:proposed_method:complexity}. As such, we expect the performance of $\lambdamart$ to degrade at a higher pace than $\xendcgmart$ as we inject more noise into the dataset. We put this hypothesis to the test through two types of experiments.

In the first series of experiments, we focus on the effect of enlarging the document list per training example by the addition of noise. In particular, we augment document lists for training examples with negative documents using the following protocol. For every training example $(\bm{x}, \bm{y})$, we sample from the collection of all documents in the training set excluding $\bm{x}$ to form $\bm{x}^{\prime} = \{\overline{x}\, |\, (\overline{x},\, \overline{y}) \sim (\overline{\bm{x}},\, \overline{\bm{y}}),\, (\overline{\bm{x}},\, \overline{\bm{y}}) \sim \Psi \setminus (\bm{x},\, \bm{y})\}$). Subsequently, we augment $\bm{x}$ by adding $\bm{x}^\prime$ as non-relevant documents: $(\bm{x} \oplus \bm{x}^{\prime}, \bm{y} \oplus \bm{0})$, where $\oplus$ denotes concatenation. Finally, we train models on the resulting training set and evaluate on the (unmodified) test set. As before, we repeat this experiment 100 times.

We illustrate NDCG@5 on the test sets averaged over 100 trials and for various degrees of augmentation in Figures~\ref{fig:robustness:augmented_list_web30k} and~\ref{fig:robustness:augmented_list_yahoo}. The trend confirms our hypothesis: On both datasets, the performance of $\lambdamart$ degrades more severely as more noise is added to the training set, increasing the number of documents per example ($m$). This effect is more pronounced on the Yahoo! dataset where $m$ is on average small. We note that the increase in NDCG@5 of $\xendcgmart$ from the 40\% mark to 60\% on Web30K is not statistically significant.

In another series of experiments we perturb relevance labels in the training set. To that end, for each training example $(\bm{x}, \bm{y})$, we randomly choose a subset of its documents and set their labels (independently) to 0 through 4 with decreasing probabilities: $p(0)=.5,\, p(1)=.2,\, p(2)=.15,\, p(3)=.1,\, p(4)=.05$. We train models on the perturbed training set and evaluate on the (unmodified) test set. As before, we repeat this experiment 100 times.

The results are shown in Figures~\ref{fig:robustness:noise_web30k} and~\ref{fig:robustness:noise_yahoo}. We observe once more that $\lambdamart$'s performance degrades more rapidly with more noise. This behavior is more pronounced on Web30K.

\subsection{Robustness to Noise in Simulated Clicks}
In Section~\ref{sec:experiments:robustness}, we examined the behavior of $\lambdamart$ and $\xendcgmart$ in the presence of noise on datasets with explicit relevance judgments. In this section, we provide an analysis of the robustness of the two algorithms on a simulated click dataset where noise occurs more naturally (e.g., where a user clicks a non-relevant document by accident).

We follow the procedure proposed in~\cite{Joachims:wsdm:2017} to simulate a user in the cascade click model~\cite{Craswell:wsdm:2008}. In the cascade click model, when presented with a ranked list of documents, a user scans the list sequentially from the top and clicks a document according to a \emph{click probability} distribution---the probability that a document is clicked given its relevance label. We assume the user is persistent in that they continue to examine the list until either a document is clicked or they reach the end of the list.

We construct click datasets as follows. We first create training and validation splits using the procedure of Section~\ref{sec:experiments:datasets}. Given a training (or validation) example $(\bm{x}, \bm{y})$ consisting of $m$ documents and relevance labels, we shuffle its elements and sequentially scan the resulting list to produce clicks using the cascade click model. We stop at the first occurrence of a click and return the list up to the first click as an ``impression.'' We create 10 impressions per training example to form our click dataset. Finally, we train ranking models on the click dataset and evaluate on the original (non-click) test set. We repeat this experiment 20 times and measure mean NDCG.

In our experiments, we adjust the click probability of non-relevant documents to simulate noise in the training set. We begin with click probabilities set to $(.05, .3, .5, .7, .95)$ for relevance labels 0 through 4, respectively. That is, in this setting, a non-relevant document is clicked 5\% of the time. In subsequent experiments, we increase the click probability of non-relevant documents by $.05$.

The results of our experiments on Web30K and Yahoo! are plotted in Figure~\ref{apx:fig:robustness_click_prob}.  Clearly, the performance of $\xendcgmart$ on the test sets is consistently better than $\lambdamart$ for all levels of noise in the training set. We note that all differences are statistically significant according to a paired \emph{t}-test ($\alpha=0.01$). Additionally, as with previous experiments, the performance of $\xendcgmart$ is more robust to noise: its performance degrades more slowly than $\lambdamart$.

\begin{figure}
\begin{center}
\centerline{
\subfloat[Web30K]{\includegraphics[height=1.35in]{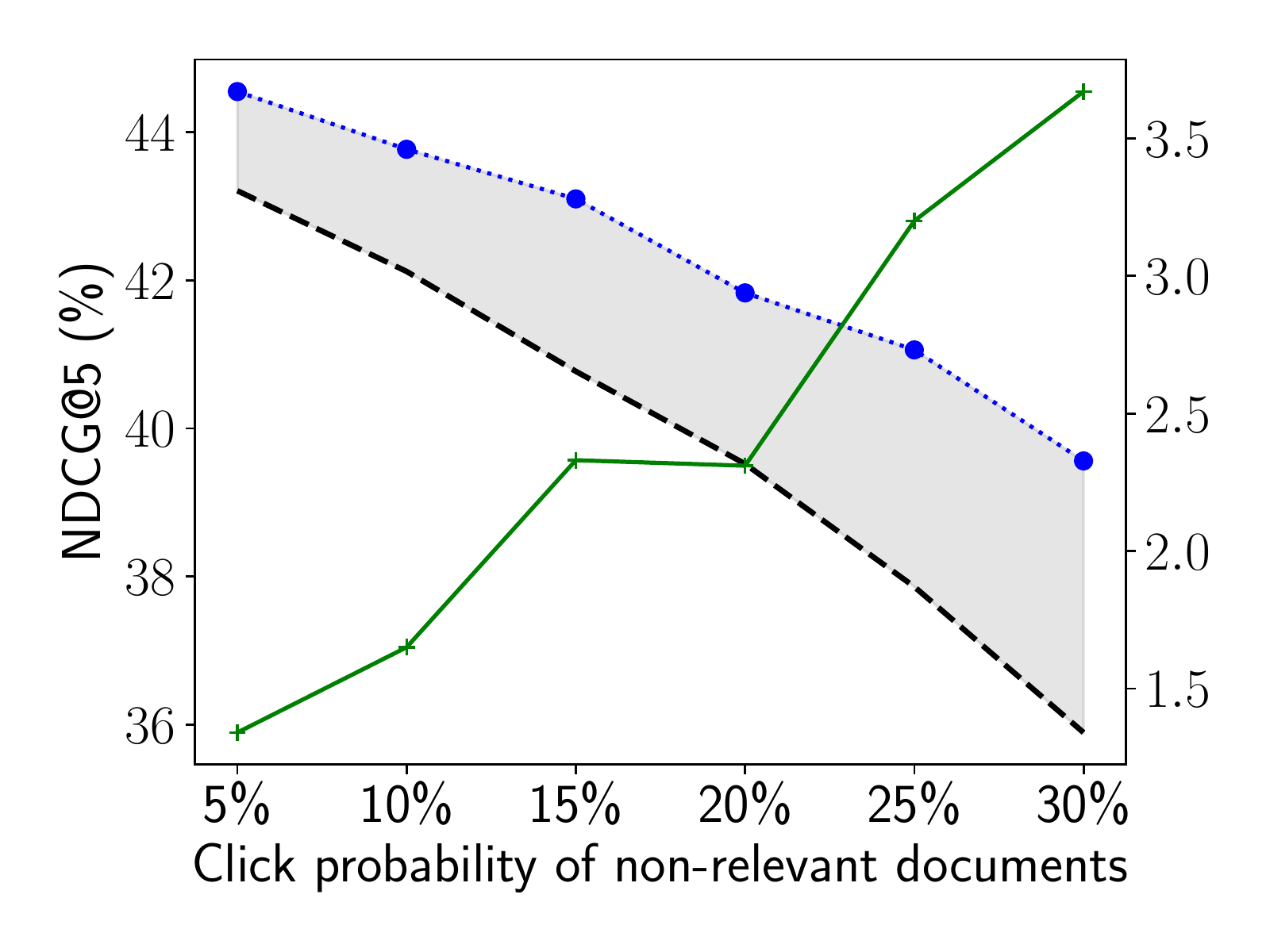}}
\hspace{-1em}
\subfloat[Yahoo!]{\includegraphics[height=1.35in]{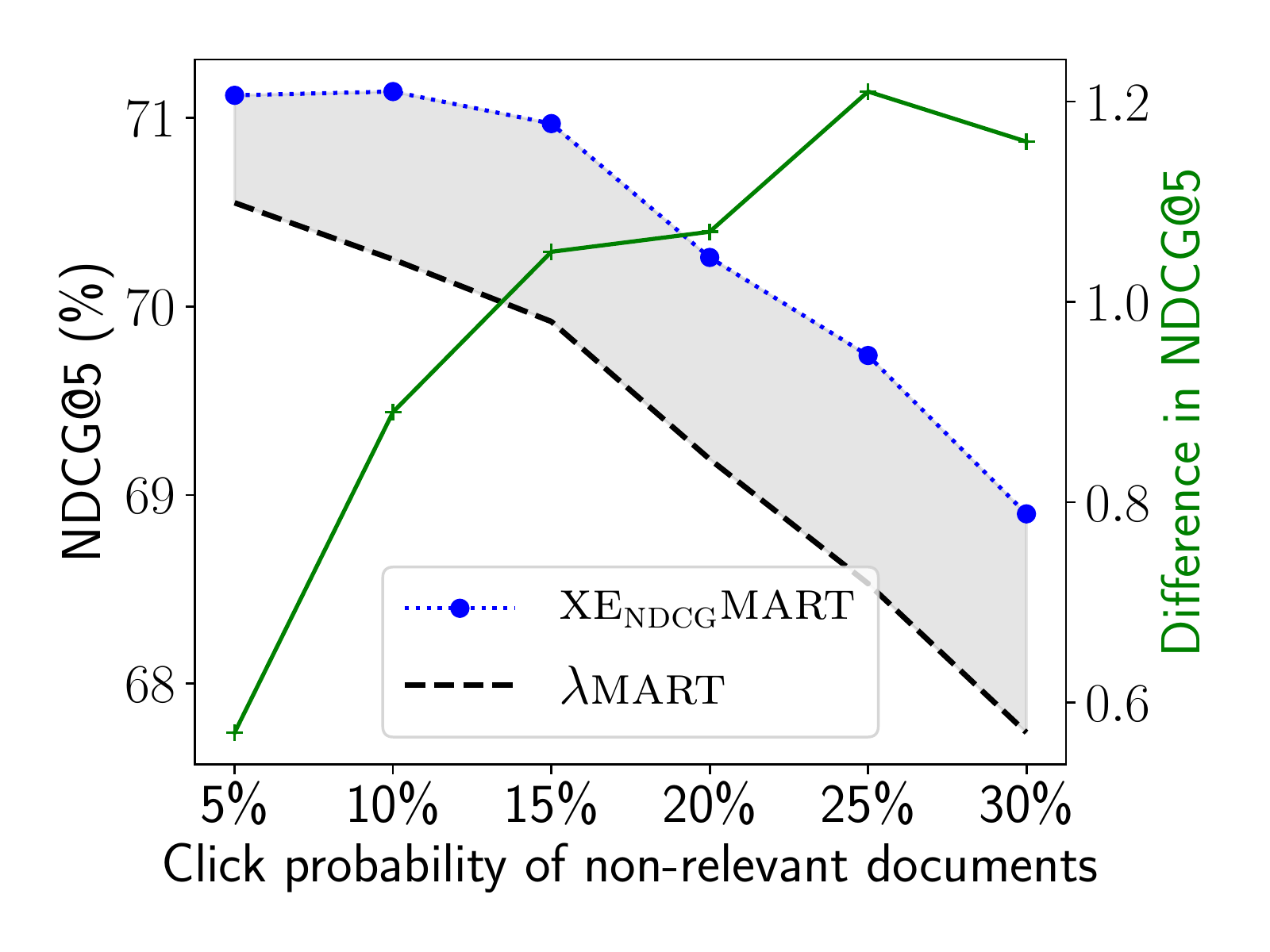}}
}
\caption{Mean NDCG@5 on test sets, averaged over 20 trials. In each trial, training and validation sets are turned into clicks using the cascade click model and a random base ranker. The horizontal axis indicates the click probability of non-relevant documents. The solid (green) lines show the difference in NDCG@5 between the two models.}
\label{apx:fig:robustness_click_prob}
\end{center}
\end{figure}

%% file: conclusion.tex
\section{Conclusion}\label{sec:conclusion}
In this work, we presented a novel ``listwise'' learning-to-rank loss function, $\xendcg$, that, unlike existing methods bounds NDCG---a popular ranking metric---in a general setting. We contrasted our proposed loss function with $\lambdamart$ and showed its superior theoretical properties. In particular, we showed that the loss function optimized by $\lambdamart$ (if it exists), has a higher complexity with a Lipschitz constant that is a function of the number of documents, $m$. In contrast, the complexity of $\xendcg$ is invariant to $m$.

Furthermore, we proposed a model that optimizes $\xendcg$ to learn an ensemble of gradient-boosted decision trees which we refer to as $\xendcgmart$. Through extensive experiments on two benchmark learning-to-rank datasets, we demonstrated that our proposed method performs better than ListNet and $\lambdamart$ in terms of quality and robustness. We showed that, $\xendcgmart$ is less sensitive to the number of documents and is more robust in the presence of noise. Finally, our experiments suggest that the performance gap between $\xendcgmart$ and $\lambdamart$ widens if we constrain the size of the learned ensemble. Better performance with fewer trees is important for latency-sensitive applications.

As a future direction, we are interested in an examination of the bound and its effect on the convergence and consistency of $\xendcg$. In particular, in this work, we treated $\bm{\gamma}$'s as hyperparameters. However, more effective strategies for solving $\bm{\gamma}$'s and obtaining tighter bounds during boosting remain unexplored. Furthermore, given its robustness to label noise (implicit and explicit), we are also interested in studying $\xendcg$ in an online learning setting.

%% file: appendix.tex
\section{Appendix}
\subsection{Proof of Claim~\ref{claim:hessian_nonsingular}}

Using $\rho_r = \rho(f_r) = e^{f_r}/(\sum{e^{f_j} + \epsilon})$ to denote the score probability of the $r^\textup{th}$ document, the Hessian can be written as follows:
\begin{equation*}
    H_{ij} =
    \begin{cases}
      \rho_i (1 - \rho_i), & i = j \\
      -\rho_i \rho_j, & i \neq j
    \end{cases}
\end{equation*}

\begin{claim*}
The Hessian, as defined above, is positive definite.
\end{claim*}
\begin{proof}
We first prove that $H$ is strictly diagonally dominant. By definition, a square matrix $A$ is strictly diagonally dominant if the following holds for all $i$: $|A_{ii}| > \sum_{j \neq i} |A_{ij}|$. Observe that:
\begin{align*}
    |H_{kk}| &= \rho_k(1- \rho_k) = \rho_k (1 - \frac{e^{f_k}}{\sum e^{f_j} + \epsilon}) \\
    &= \rho_k \frac{ \sum_{j \neq k} e^{f_j} + \epsilon }{ \sum e^{f_j} + \epsilon }
      > \rho_k \sum_{j \neq k} \rho_j = \sum_{j \neq k} |H_{kj}|.
\end{align*}

Using this property, we now prove nonsingularity of $H$ by contradiction. Assume there exists a vector $\bm{u} \neq \bm{0}$ such that $H\bm{u} = \bm{0}$. Let $i$ be the index of the $u_i$ with the largest magnitude: $i = \argmax_{i} |u_i|$. Then:
\begin{align*}
    \sum_{j} H_{ij} u_j = 0 & \Rightarrow H_{ii} u_i = - \sum_{j \neq i} H_{ij} u_j
    \overset{u_i \neq 0}{\Rightarrow} H_{ii} = - \sum_{j \neq i} \frac{u_j}{u_i} H_{ij} \\
    & \Rightarrow |H_{ii}| \leq \sum_{j \neq i} |\frac{u_j}{u_i} H_{ij}|
    \Rightarrow |H_{ii}| \leq \sum_{j \neq i} |H_{ij}|,
\end{align*}
which is a contradiction. This concludes the proof for nonsingularity of the Hessian, which is already sufficient for subsequent results. However, as a consequence of the Gershgorin circle theorem it can further be shown that, because the diagonal elements of $H$ are strictly positive, $H$ is positive definite.
\end{proof}

\subsection{Proof of Claim~\ref{claim:spectral_radius}}

Use $\rho_r = \rho(f_r) = e^{f_r}/(\sum{e^{f_j} + \epsilon})$ to denote the score probability of the $r^\textup{th}$ document. The nonnegative, square matrix $S$ in Claim~\ref{claim:spectral_radius} is defined as follows:
\begin{equation*}
    S_{ij} =
    \begin{cases}
    0, & i=j\\
    \rho_j/(1 - \rho_i), & i \neq j
    \end{cases}.
\end{equation*}

\begin{claim*}
The spectral radius of $S$ is strictly less than 1.
\end{claim*}
\begin{proof}
Note that, for all eigenvalues $\lambda_i,\, 1\leq i \leq n$ of an $n\times n$ matrix $A$, their corresponding eigenvectors $\bm{u}_i$, and for any induced operator norm $\| \cdot \|$ we have that:
\begin{align*}
    \| A \| = \sup_{\bm{x}} \frac{\|A\bm{x}\|}{\|\bm{x}\|} \geq \frac{\|A\bm{u}_i\|}{\|\bm{u}_i\|} = \frac{\|\lambda_i\bm{u}_i\|}{\|\bm{u}_i\|} = |\lambda_i|,\, \forall i.
\end{align*}
This is, in particular, true for the infinity norm:
\begin{equation*}
    \|A\|_\infty = \max_i \sum_{j} |A_{ij}|.
\end{equation*}

The inequality above holds for the spectral radius of $A$ which is defined as the largest absolute value of $A$'s eigenvalues: $\max_i |\lambda_i|$. Therefore, we have that the spectral radius of $S$ is bounded above by:
\begin{align*}
    \max_i \sum_{j} |S_{ij}| &= \max_i \sum_j \frac{\rho_j}{1 - \rho_i}
    = \max_i \frac{\sum_{j \neq i} \rho_j}{1 - \rho_i} \\
    &= \max_i \frac{1 - \rho_i - \epsilon^\prime}{1 - \rho_i}
    = \max_i 1 - \frac{\epsilon^\prime}{1 - \rho_i} < 1,
\end{align*}
where $\epsilon^\prime = \epsilon/\sum e^{f_k} + \epsilon$. That completes the proof.
\end{proof}

%% file: main.bbl

\begin{thebibliography}{36}


\ifx \showCODEN    \undefined \def \showCODEN     #1{\unskip}     \fi
\ifx \showDOI      \undefined \def \showDOI       #1{#1}\fi
\ifx \showISBNx    \undefined \def \showISBNx     #1{\unskip}     \fi
\ifx \showISBNxiii \undefined \def \showISBNxiii  #1{\unskip}     \fi
\ifx \showISSN     \undefined \def \showISSN      #1{\unskip}     \fi
\ifx \showLCCN     \undefined \def \showLCCN      #1{\unskip}     \fi
\ifx \shownote     \undefined \def \shownote      #1{#1}          \fi
\ifx \showarticletitle \undefined \def \showarticletitle #1{#1}   \fi
\ifx \showURL      \undefined \def \showURL       {\relax}        \fi
\providecommand\bibfield[2]{#2}
\providecommand\bibinfo[2]{#2}
\providecommand\natexlab[1]{#1}
\providecommand\showeprint[2][]{arXiv:#2}

\bibitem[\protect\citeauthoryear{Bruch, Han, Bendersky, and Najork}{Bruch
  et~al\mbox{.}}{2020}]%
        {Bruch:wsdm:2020}
\bibfield{author}{\bibinfo{person}{Sebastian Bruch}, \bibinfo{person}{Shuguang
  Han}, \bibinfo{person}{Michael Bendersky}, {and} \bibinfo{person}{Marc
  Najork}.} \bibinfo{year}{2020}\natexlab{}.
\newblock \showarticletitle{A Stochastic Treatment of Learning to Rank Scoring
  Functions}. In \bibinfo{booktitle}{\emph{Proceedings of the 13th
  International Conference on Web Search and Data Mining}}.
  \bibinfo{pages}{61–69}.
\newblock


\bibitem[\protect\citeauthoryear{Bruch, Wang, Bendersky, and Najork}{Bruch
  et~al\mbox{.}}{2019a}]%
        {BruchAnalysisICTIR2019}
\bibfield{author}{\bibinfo{person}{Sebastian Bruch}, \bibinfo{person}{Xuanhui
  Wang}, \bibinfo{person}{Mike Bendersky}, {and} \bibinfo{person}{Marc
  Najork}.} \bibinfo{year}{2019}\natexlab{a}.
\newblock \showarticletitle{An Analysis of the Softmax Cross Entropy Loss for
  Learning-to-Rank with Binary Relevance}. In
  \bibinfo{booktitle}{\emph{Proceedings of the 2019 ACM SIGIR International
  Conference on the Theory of Information Retrieval}}.
\newblock


\bibitem[\protect\citeauthoryear{Bruch, Zoghi, Bendersky, and Najork}{Bruch
  et~al\mbox{.}}{2019b}]%
        {BruchApproxSIGIR2019}
\bibfield{author}{\bibinfo{person}{Sebastian Bruch}, \bibinfo{person}{Masrour
  Zoghi}, \bibinfo{person}{Mike Bendersky}, {and} \bibinfo{person}{Marc
  Najork}.} \bibinfo{year}{2019}\natexlab{b}.
\newblock \showarticletitle{Revisiting Approximate Metric Optimization in the
  Age of Deep Neural Networks}. In \bibinfo{booktitle}{\emph{Proceedings of the
  42nd International ACM SIGIR Conference on Research and Development in
  Information Retrieval}}.
\newblock


\bibitem[\protect\citeauthoryear{Burges, Shaked, Renshaw, Lazier, Deeds,
  Hamilton, and Hullender}{Burges et~al\mbox{.}}{2005}]%
        {burges2005learning}
\bibfield{author}{\bibinfo{person}{Chris Burges}, \bibinfo{person}{Tal Shaked},
  \bibinfo{person}{Erin Renshaw}, \bibinfo{person}{Ari Lazier},
  \bibinfo{person}{Matt Deeds}, \bibinfo{person}{Nicole Hamilton}, {and}
  \bibinfo{person}{Greg Hullender}.} \bibinfo{year}{2005}\natexlab{}.
\newblock \showarticletitle{Learning to rank using gradient descent}. In
  \bibinfo{booktitle}{\emph{Proceedings of the 22nd International Conference on
  Machine Learning}}. \bibinfo{pages}{89--96}.
\newblock


\bibitem[\protect\citeauthoryear{Burges}{Burges}{2010}]%
        {burges2010ranknet}
\bibfield{author}{\bibinfo{person}{Christopher~J.C. Burges}.}
  \bibinfo{year}{2010}\natexlab{}.
\newblock \bibinfo{booktitle}{\emph{From {RankNet} to {LambdaRank} to
  {LambdaMART}: An Overview}}.
\newblock \bibinfo{type}{{T}echnical {R}eport} MSR-TR-2010-82.
  \bibinfo{institution}{Microsoft Research}.
\newblock


\bibitem[\protect\citeauthoryear{Cao, Qin, Liu, Tsai, and Li}{Cao
  et~al\mbox{.}}{2007}]%
        {cao2007learning}
\bibfield{author}{\bibinfo{person}{Zhe Cao}, \bibinfo{person}{Tao Qin},
  \bibinfo{person}{Tie-Yan Liu}, \bibinfo{person}{Ming-Feng Tsai}, {and}
  \bibinfo{person}{Hang Li}.} \bibinfo{year}{2007}\natexlab{}.
\newblock \showarticletitle{Learning to rank: from pairwise approach to
  listwise approach}. In \bibinfo{booktitle}{\emph{Proceedings of the 24th
  International Conference on Machine Learning}}. \bibinfo{pages}{129--136}.
\newblock


\bibitem[\protect\citeauthoryear{Chapelle and Chang}{Chapelle and
  Chang}{2011}]%
        {chapelle2011yahoo}
\bibfield{author}{\bibinfo{person}{Olivier Chapelle} {and} \bibinfo{person}{Yi
  Chang}.} \bibinfo{year}{2011}\natexlab{}.
\newblock \showarticletitle{Yahoo! Learning to Rank Challenge Overview}.
  \bibinfo{pages}{1--24}.
\newblock


\bibitem[\protect\citeauthoryear{Chapelle, Metzler, Zhang, and
  Grinspan}{Chapelle et~al\mbox{.}}{2009}]%
        {Chapelle:ERR:2009}
\bibfield{author}{\bibinfo{person}{Olivier Chapelle}, \bibinfo{person}{Donald
  Metzler}, \bibinfo{person}{Ya Zhang}, {and} \bibinfo{person}{Pierre
  Grinspan}.} \bibinfo{year}{2009}\natexlab{}.
\newblock \showarticletitle{Expected Reciprocal Rank for Graded Relevance}. In
  \bibinfo{booktitle}{\emph{Proceedings of the 18th ACM Conference on
  Information and Knowledge Management}}. \bibinfo{pages}{621--630}.
\newblock


\bibitem[\protect\citeauthoryear{Chapelle and Wu}{Chapelle and Wu}{2010}]%
        {Chapelle:IRJ:2010}
\bibfield{author}{\bibinfo{person}{Olivier Chapelle} {and}
  \bibinfo{person}{Mingrui Wu}.} \bibinfo{year}{2010}\natexlab{}.
\newblock \showarticletitle{Gradient Descent Optimization of Smoothed
  Information Retrieval Metrics}.
\newblock \bibinfo{journal}{\emph{Information Retrieval}} \bibinfo{volume}{13},
  \bibinfo{number}{3} (\bibinfo{date}{June} \bibinfo{year}{2010}),
  \bibinfo{pages}{216--235}.
\newblock


\bibitem[\protect\citeauthoryear{Cl{\'e}men\c{c}on and
  Vayatis}{Cl{\'e}men\c{c}on and Vayatis}{2008}]%
        {Clemencon:NeurIPS:2008}
\bibfield{author}{\bibinfo{person}{St{\'e}phan Cl{\'e}men\c{c}on} {and}
  \bibinfo{person}{Nicolas Vayatis}.} \bibinfo{year}{2008}\natexlab{}.
\newblock \showarticletitle{Empirical Performance Maximization for Linear Rank
  Statistics}. In \bibinfo{booktitle}{\emph{Proceedings of the 21st
  International Conference on Neural Information Processing Systems}}.
  \bibinfo{pages}{305--312}.
\newblock


\bibitem[\protect\citeauthoryear{Craswell, Zoeter, Taylor, and Ramsey}{Craswell
  et~al\mbox{.}}{2008}]%
        {Craswell:wsdm:2008}
\bibfield{author}{\bibinfo{person}{Nick Craswell}, \bibinfo{person}{Onno
  Zoeter}, \bibinfo{person}{Michael Taylor}, {and} \bibinfo{person}{Bill
  Ramsey}.} \bibinfo{year}{2008}\natexlab{}.
\newblock \showarticletitle{An Experimental Comparison of Click Position-bias
  Models}. In \bibinfo{booktitle}{\emph{Proceedings of the 2008 International
  Conference on Web Search and Data Mining}}. \bibinfo{pages}{87--94}.
\newblock


\bibitem[\protect\citeauthoryear{Cucker and Smale}{Cucker and Smale}{2002}]%
        {Cucker02onthe}
\bibfield{author}{\bibinfo{person}{Felipe Cucker} {and} \bibinfo{person}{Steve
  Smale}.} \bibinfo{year}{2002}\natexlab{}.
\newblock \showarticletitle{On the mathematical foundations of learning}.
\newblock \bibinfo{journal}{\emph{Bull. Amer. Math. Soc.}}
  \bibinfo{volume}{39} (\bibinfo{year}{2002}), \bibinfo{pages}{1--49}.
\newblock


\bibitem[\protect\citeauthoryear{Friedman}{Friedman}{2001}]%
        {friedman2001greedy}
\bibfield{author}{\bibinfo{person}{Jerome~H Friedman}.}
  \bibinfo{year}{2001}\natexlab{}.
\newblock \showarticletitle{Greedy function approximation: a gradient boosting
  machine}.
\newblock \bibinfo{journal}{\emph{Annals of Statistics}} \bibinfo{volume}{29},
  \bibinfo{number}{5} (\bibinfo{year}{2001}), \bibinfo{pages}{1189--1232}.
\newblock


\bibitem[\protect\citeauthoryear{J{\"a}rvelin and
  Kek{\"a}l{\"a}inen}{J{\"a}rvelin and Kek{\"a}l{\"a}inen}{2002}]%
        {jarvelin2002cumulated}
\bibfield{author}{\bibinfo{person}{Kalervo J{\"a}rvelin} {and}
  \bibinfo{person}{Jaana Kek{\"a}l{\"a}inen}.} \bibinfo{year}{2002}\natexlab{}.
\newblock \showarticletitle{Cumulated gain-based evaluation of IR techniques}.
\newblock \bibinfo{journal}{\emph{ACM Transactions on Information Systems}}
  \bibinfo{volume}{20}, \bibinfo{number}{4} (\bibinfo{year}{2002}),
  \bibinfo{pages}{422--446}.
\newblock


\bibitem[\protect\citeauthoryear{Joachims}{Joachims}{2006}]%
        {joachims2006training}
\bibfield{author}{\bibinfo{person}{Thorsten Joachims}.}
  \bibinfo{year}{2006}\natexlab{}.
\newblock \showarticletitle{Training linear SVMs in linear time}. In
  \bibinfo{booktitle}{\emph{Proceedings of the 12th ACM SIGKDD International
  Conference on Knowledge Discovery and Data Mining}}.
  \bibinfo{pages}{217--226}.
\newblock


\bibitem[\protect\citeauthoryear{Joachims, Swaminathan, and Schnabel}{Joachims
  et~al\mbox{.}}{2017}]%
        {Joachims:wsdm:2017}
\bibfield{author}{\bibinfo{person}{Thorsten Joachims}, \bibinfo{person}{Adith
  Swaminathan}, {and} \bibinfo{person}{Tobias Schnabel}.}
  \bibinfo{year}{2017}\natexlab{}.
\newblock \showarticletitle{Unbiased Learning-to-Rank with Biased Feedback}. In
  \bibinfo{booktitle}{\emph{Proceedings of the 10th ACM International
  Conference on Web Search and Data Mining}}. \bibinfo{pages}{781--789}.
\newblock


\bibitem[\protect\citeauthoryear{Ke, Meng, Finley, Wang, Chen, Ma, Ye, and
  Liu}{Ke et~al\mbox{.}}{2017}]%
        {lightgbm2017nips}
\bibfield{author}{\bibinfo{person}{Guolin Ke}, \bibinfo{person}{Qi Meng},
  \bibinfo{person}{Thomas Finley}, \bibinfo{person}{Taifeng Wang},
  \bibinfo{person}{Wei Chen}, \bibinfo{person}{Weidong Ma},
  \bibinfo{person}{Qiwei Ye}, {and} \bibinfo{person}{Tie-Yan Liu}.}
  \bibinfo{year}{2017}\natexlab{}.
\newblock \showarticletitle{LightGBM: A Highly Efficient Gradient Boosting
  Decision Tree}.
\newblock In \bibinfo{booktitle}{\emph{Advances in Neural Information
  Processing Systems 30}}. \bibinfo{pages}{3146--3154}.
\newblock


\bibitem[\protect\citeauthoryear{Lan, Liu, Ma, and Li}{Lan
  et~al\mbox{.}}{2009}]%
        {Lan:ICML:2009}
\bibfield{author}{\bibinfo{person}{Yanyan Lan}, \bibinfo{person}{Tie-Yan Liu},
  \bibinfo{person}{Zhiming Ma}, {and} \bibinfo{person}{Hang Li}.}
  \bibinfo{year}{2009}\natexlab{}.
\newblock \showarticletitle{Generalization Analysis of Listwise
  Learning-to-rank Algorithms}. In \bibinfo{booktitle}{\emph{Proceedings of the
  26th Annual International Conference on Machine Learning}}.
  \bibinfo{pages}{577--584}.
\newblock


\bibitem[\protect\citeauthoryear{Lan, Zhu, Guo, Niu, and Cheng}{Lan
  et~al\mbox{.}}{2014}]%
        {Lan:uai:2014}
\bibfield{author}{\bibinfo{person}{Yanyan Lan}, \bibinfo{person}{Yadong Zhu},
  \bibinfo{person}{Jiafeng Guo}, \bibinfo{person}{Shuzi Niu}, {and}
  \bibinfo{person}{Xueqi Cheng}.} \bibinfo{year}{2014}\natexlab{}.
\newblock \showarticletitle{Position-Aware ListMLE: A Sequential Learning
  Process for Ranking}. In \bibinfo{booktitle}{\emph{Proceedings of the 30th
  Conference on Uncertainty in Artificial Intelligence}}.
  \bibinfo{pages}{449–458}.
\newblock


\bibitem[\protect\citeauthoryear{Liu}{Liu}{2009}]%
        {liu2009learning}
\bibfield{author}{\bibinfo{person}{Tie-Yan Liu}.}
  \bibinfo{year}{2009}\natexlab{}.
\newblock \showarticletitle{Learning to rank for information retrieval}.
\newblock \bibinfo{journal}{\emph{Foundations and Trends in Information
  Retrieval}} \bibinfo{volume}{3}, \bibinfo{number}{3} (\bibinfo{year}{2009}),
  \bibinfo{pages}{225--331}.
\newblock


\bibitem[\protect\citeauthoryear{Metzler, Croft, and Mccallum}{Metzler
  et~al\mbox{.}}{2005}]%
        {metzler2005directMaximization}
\bibfield{author}{\bibinfo{person}{Donald~A Metzler}, \bibinfo{person}{W~Bruce
  Croft}, {and} \bibinfo{person}{Andrew Mccallum}.}
  \bibinfo{year}{2005}\natexlab{}.
\newblock \bibinfo{booktitle}{\emph{Direct maximization of rank-based metrics
  for information retrieval}}.
\newblock \bibinfo{type}{CIIR report} 429. \bibinfo{institution}{University of
  Massachusetts}.
\newblock


\bibitem[\protect\citeauthoryear{Pasumarthi, Wang, Bendersky, and
  Najork}{Pasumarthi et~al\mbox{.}}{2019}]%
        {Pasumarthi:arxiv:1910.09676}
\bibfield{author}{\bibinfo{person}{Rama~Kumar Pasumarthi},
  \bibinfo{person}{Xuanhui Wang}, \bibinfo{person}{Michael Bendersky}, {and}
  \bibinfo{person}{Marc Najork}.} \bibinfo{year}{2019}\natexlab{}.
\newblock \showarticletitle{Self-Attentive Document Interaction Networks for
  Permutation Equivariant Ranking}.
\newblock
\showeprint[arxiv]{1910.09676}


\bibitem[\protect\citeauthoryear{Qin and Liu}{Qin and Liu}{2013}]%
        {DBLP:journals/corr/QinL13}
\bibfield{author}{\bibinfo{person}{Tao Qin} {and} \bibinfo{person}{Tie-Yan
  Liu}.} \bibinfo{year}{2013}\natexlab{}.
\newblock \bibinfo{title}{Introducing {LETOR} 4.0 Datasets}.
\newblock   (\bibinfo{year}{2013}).
\newblock
\showeprint[arxiv]{1306.2597}


\bibitem[\protect\citeauthoryear{Qin, Liu, and Li}{Qin et~al\mbox{.}}{2010}]%
        {qin2010general}
\bibfield{author}{\bibinfo{person}{Tao Qin}, \bibinfo{person}{Tie-Yan Liu},
  {and} \bibinfo{person}{Hang Li}.} \bibinfo{year}{2010}\natexlab{}.
\newblock \showarticletitle{A general approximation framework for direct
  optimization of information retrieval measures}.
\newblock \bibinfo{journal}{\emph{Information Retrieval}} \bibinfo{volume}{13},
  \bibinfo{number}{4} (\bibinfo{year}{2010}), \bibinfo{pages}{375--397}.
\newblock


\bibitem[\protect\citeauthoryear{Ravikumar, Tewari, and Yang}{Ravikumar
  et~al\mbox{.}}{2011}]%
        {pmlr-v15-ravikumar11a}
\bibfield{author}{\bibinfo{person}{Pradeep Ravikumar}, \bibinfo{person}{Ambuj
  Tewari}, {and} \bibinfo{person}{Eunho Yang}.}
  \bibinfo{year}{2011}\natexlab{}.
\newblock \showarticletitle{On NDCG Consistency of Listwise Ranking Methods}.
  In \bibinfo{booktitle}{\emph{Proceedings of the Fourteenth International
  Conference on Artificial Intelligence and Statistics}},
  Vol.~\bibinfo{volume}{15}. \bibinfo{publisher}{PMLR},
  \bibinfo{pages}{618--626}.
\newblock


\bibitem[\protect\citeauthoryear{Rudin}{Rudin}{2009}]%
        {Rudin:JMLR:2009}
\bibfield{author}{\bibinfo{person}{Cynthia Rudin}.}
  \bibinfo{year}{2009}\natexlab{}.
\newblock \showarticletitle{The P-Norm Push: A Simple Convex Ranking Algorithm
  That Concentrates at the Top of the List}.
\newblock \bibinfo{journal}{\emph{Journal of Machine Learning Research}}
  \bibinfo{volume}{10} (\bibinfo{date}{Dec.} \bibinfo{year}{2009}),
  \bibinfo{pages}{2233--2271}.
\newblock


\bibitem[\protect\citeauthoryear{Rudin and Wang}{Rudin and Wang}{2018}]%
        {RuWa:aistats:2018}
\bibfield{author}{\bibinfo{person}{Cynthia Rudin} {and} \bibinfo{person}{Yining
  Wang}.} \bibinfo{year}{2018}\natexlab{}.
\newblock \showarticletitle{Direct Learning to Rank and Rerank}. In
  \bibinfo{booktitle}{\emph{Proceedings of Artificial Intelligence and
  Statistics {AISTATS}}}.
\newblock


\bibitem[\protect\citeauthoryear{Taylor, Guiver, Robertson, and Minka}{Taylor
  et~al\mbox{.}}{2008}]%
        {Taylor+al:2008}
\bibfield{author}{\bibinfo{person}{Michael Taylor}, \bibinfo{person}{John
  Guiver}, \bibinfo{person}{Stephen Robertson}, {and} \bibinfo{person}{Tom
  Minka}.} \bibinfo{year}{2008}\natexlab{}.
\newblock \showarticletitle{SoftRank: Optimizing Non-smooth Rank Metrics}. In
  \bibinfo{booktitle}{\emph{Proceedings of the 1st International Conference on
  Web Search and Data Mining}}. \bibinfo{pages}{77--86}.
\newblock


\bibitem[\protect\citeauthoryear{Tewari and Chaudhuri}{Tewari and
  Chaudhuri}{2015}]%
        {Tewari:ICML2015}
\bibfield{author}{\bibinfo{person}{Ambuj Tewari} {and} \bibinfo{person}{Sougata
  Chaudhuri}.} \bibinfo{year}{2015}\natexlab{}.
\newblock \showarticletitle{Generalization Error Bounds for Learning to Rank:
  Does the Length of Document Lists Matter?}. In
  \bibinfo{booktitle}{\emph{Proceedings of the 32nd International Conference on
  Machine Learning}}. \bibinfo{pages}{315--323}.
\newblock


\bibitem[\protect\citeauthoryear{Volkovs and Zemel}{Volkovs and Zemel}{[n.
  d.]}]%
        {volkovs:icml:2009}
\bibfield{author}{\bibinfo{person}{Maksims~N. Volkovs} {and}
  \bibinfo{person}{Richard~S. Zemel}.} \bibinfo{year}{[n. d.]}\natexlab{}.
\newblock \showarticletitle{BoltzRank: learning to maximize expected ranking
  gain}. In \bibinfo{booktitle}{\emph{Proceedings of the 26th Annual
  International Conference on Machine Learning}}. \bibinfo{pages}{1089--1096}.
\newblock


\bibitem[\protect\citeauthoryear{Wang, Li, Golbandi, Bendersky, and
  Najork}{Wang et~al\mbox{.}}{2018}]%
        {WangLambdaLoss}
\bibfield{author}{\bibinfo{person}{Xuanhui Wang}, \bibinfo{person}{Cheng Li},
  \bibinfo{person}{Nadav Golbandi}, \bibinfo{person}{Michael Bendersky}, {and}
  \bibinfo{person}{Marc Najork}.} \bibinfo{year}{2018}\natexlab{}.
\newblock \showarticletitle{The LambdaLoss Framework for Ranking Metric
  Optimization}. In \bibinfo{booktitle}{\emph{Proceedings of the 27th ACM
  International Conference on Information and Knowledge Management}}.
  \bibinfo{pages}{1313--1322}.
\newblock


\bibitem[\protect\citeauthoryear{Wu, Burges, Svore, and Gao}{Wu
  et~al\mbox{.}}{2010}]%
        {wu2010adapting}
\bibfield{author}{\bibinfo{person}{Qiang Wu}, \bibinfo{person}{Christopher~JC
  Burges}, \bibinfo{person}{Krysta~M Svore}, {and} \bibinfo{person}{Jianfeng
  Gao}.} \bibinfo{year}{2010}\natexlab{}.
\newblock \showarticletitle{Adapting boosting for information retrieval
  measures}.
\newblock \bibinfo{journal}{\emph{Information Retrieval}} \bibinfo{volume}{13},
  \bibinfo{number}{3} (\bibinfo{year}{2010}), \bibinfo{pages}{254--270}.
\newblock


\bibitem[\protect\citeauthoryear{Xia, Liu, Wang, Zhang, and Li}{Xia
  et~al\mbox{.}}{2008}]%
        {xia2008listwise}
\bibfield{author}{\bibinfo{person}{Fen Xia}, \bibinfo{person}{Tie-Yan Liu},
  \bibinfo{person}{Jue Wang}, \bibinfo{person}{Wensheng Zhang}, {and}
  \bibinfo{person}{Hang Li}.} \bibinfo{year}{2008}\natexlab{}.
\newblock \showarticletitle{Listwise approach to learning to rank: theory and
  algorithm}. In \bibinfo{booktitle}{\emph{Proceedings of the 25th
  International Conference on Machine Learning}}. \bibinfo{pages}{1192--1199}.
\newblock


\bibitem[\protect\citeauthoryear{Xu and Li}{Xu and Li}{2007}]%
        {Jun+Hang:2007}
\bibfield{author}{\bibinfo{person}{Jun Xu} {and} \bibinfo{person}{Hang Li}.}
  \bibinfo{year}{2007}\natexlab{}.
\newblock \showarticletitle{AdaRank: A Boosting Algorithm for Information
  Retrieval}. In \bibinfo{booktitle}{\emph{Proceedings of the 30th Annual
  International ACM SIGIR Conference on Research and Development in Information
  Retrieval}}. \bibinfo{pages}{391--398}.
\newblock


\bibitem[\protect\citeauthoryear{Xu, Liu, Lu, Li, and Ma}{Xu
  et~al\mbox{.}}{2008}]%
        {Xu2008DirectlyOptimizingLTR}
\bibfield{author}{\bibinfo{person}{Jun Xu}, \bibinfo{person}{Tie-Yan Liu},
  \bibinfo{person}{Min Lu}, \bibinfo{person}{Hang Li}, {and}
  \bibinfo{person}{Wei-Ying Ma}.} \bibinfo{year}{2008}\natexlab{}.
\newblock \showarticletitle{Directly Optimizing Evaluation Measures in Learning
  to Rank}. In \bibinfo{booktitle}{\emph{Proceedings of the 31st Annual
  International ACM SIGIR Conference on Research and Development in Information
  Retrieval}}. \bibinfo{pages}{107–114}.
\newblock


\bibitem[\protect\citeauthoryear{Zhuang, Wang, Bendersky, Grushetsky, Wu,
  Mitrichev, Sterling, Bell, Ravina, and Qian}{Zhuang et~al\mbox{.}}{2020}]%
        {Zhuang:arxiv:2005.02553}
\bibfield{author}{\bibinfo{person}{Honglei Zhuang}, \bibinfo{person}{Xuanhui
  Wang}, \bibinfo{person}{Michael Bendersky}, \bibinfo{person}{Alexander
  Grushetsky}, \bibinfo{person}{Yonghui Wu}, \bibinfo{person}{Petr Mitrichev},
  \bibinfo{person}{Ethan Sterling}, \bibinfo{person}{Nathan Bell},
  \bibinfo{person}{Walker Ravina}, {and} \bibinfo{person}{Hai Qian}.}
  \bibinfo{year}{2020}\natexlab{}.
\newblock \showarticletitle{Interpretable Learning-to-Rank with Generalized
  Additive Models}.
\newblock
\showeprint[arxiv]{2005.02553}


\end{thebibliography}
